\documentclass[letterpaper, 10 pt, conference]{ieeeconf}  
\pdfoutput=1

\IEEEoverridecommandlockouts                              

\usepackage{graphicx}
\usepackage{float}
\usepackage{subfig}
\usepackage{booktabs}
\usepackage{cite}
\usepackage{caption}

\usepackage{amsmath, amsthm, amssymb}
\usepackage{mathrsfs}
\usepackage{bm}

\newcommand\scalemath[2]{\scalebox{#1}{\mbox{\ensuremath{\displaystyle #2}}}}
\newtheorem{theorem}{Theorem}

\title{\LARGE \bf
FEJ-VIRO: A Consistent First-Estimate Jacobian Visual-Inertial-Ranging Odometry
}

\author{Shenhan Jia, Yanmei Jiao, Zhuqing Zhang, Rong Xiong and Yue Wang 
\thanks{All authors are with the State Key Laboratory of Industrial Control and Technology, Zhejiang University, Hangzhou, P.R. China. Yue Wang is the corresponding author wangyue@iipc.zju.edu.cn.}
}

\begin{document}

\maketitle
\thispagestyle{empty}
\pagestyle{empty}

\begin{abstract}

In recent years, Visual-Inertial Odometry (VIO) has achieved many significant progresses. However, VIO methods suffer from localization drift over long trajectories. In this paper, we propose a First-Estimates Jacobian Visual-Inertial-Ranging Odometry (FEJ-VIRO) to reduce the localization drifts of VIO by incorporating ultra-wideband (UWB) ranging measurements into the VIO framework \textit{consistently}. Considering that the initial positions of UWB anchors are usually unavailable, we propose a long-short window structure to initialize the UWB anchors' positions as well as the covariance for state augmentation. After initialization, the FEJ-VIRO estimates the UWB anchors' positions simultaneously along with the robot poses. We further analyze the observability of the visual-inertial-ranging estimators and proved that there are \textit{four} unobservable directions in the ideal case, while one of them vanishes in the actual case due to the gain of spurious information. Based on these analyses, we leverage the FEJ technique to enforce the unobservable directions, hence reducing inconsistency of the estimator. Finally, we validate our analysis and evaluate the proposed FEJ-VIRO with both simulation and real-world experiments.

\end{abstract}

\section{INTRODUCTION}
State estimation is a fundamental capacity to enable the operation of autonomous robot systems with applications such as autonomous driving, unmanned aerial vehicles and so on. Robust and reliable state estimation is essential in real world applications as large errors in state estimation can lead to the destruction of robots. In recent years, visual inertial odometry (VIO) is attracting increasing attention because of its lightweight, accuracy, and reliability, which has been successfully applied to many real-time robotic systems\cite{ORB-SLAM3}\cite{qin2017vins}\cite{OpenVINS}.

There is significant estimation drift during a long period of time in onboard self-localization mode of VIO system. Therefore, loop-closure and bundle adjustment (BA) techniques are required to correct drift, but are often accompanied by drastic changes in estimates that introduce destabilization to the robot\cite{viral_fusion}. So in live operation of robotic systems, these techniques are usually disabled to ensure stable operation. To achieve accurate onboard self-localization, an alternative solution is to fuse the external localization information from GPS \cite{GPS_Huang} motion capture (mocap)\cite{mocap_cite}, or artificial visual markers (such as AprilTags)\cite{tags_vio}. However, GPS can only work in open spaces and mocap requires expensive and complex setup in a fixed indoor environment, neither can be rolled out to dense urban areas. The use of artificial visual markers is inflexible and it's difficult to arrange the markers to cover a large area. Therefore, incorporating Ultra Wideband (UWB) measurements is a promising solution which can be deployed in both indoor and outdoor environments at low cost. In this paper, we focus on the situation that the positions of UWB anchors are unknown and require to be estimated simultaneously during operation.

For the fusion of UWB, there are two lines of methods. The first line is the optimization-based methods which fuse the UWB data by constructing cost function to optimize\cite{viral_slam}\cite{viral_fusion}. However, such methods are computationally intensive. Due to the limited on-board computing resources on robotic systems, the lightweight fusion methods are required. Another lightweight line is the filter-based methods which use extended Kalman filter (EKF) for state estimation\cite{uwb_rough_init}. The key problem in filtering-based methods is the state modeling of UWB measurements, including the design of state vector and the evaluation of its Jacobian matrices. After conducting experiments using existing filtering-based fusion methods, we find the accuracy decrease without reasonably dealing with the uncertainty. Therefore, an explicit theory is required to guide the state modeling of fusing UWB measurements in filtering-based approaches.

\begin{figure}
  \centering 
    \includegraphics[width=1.0\columnwidth,trim=205 185 245 85,clip]{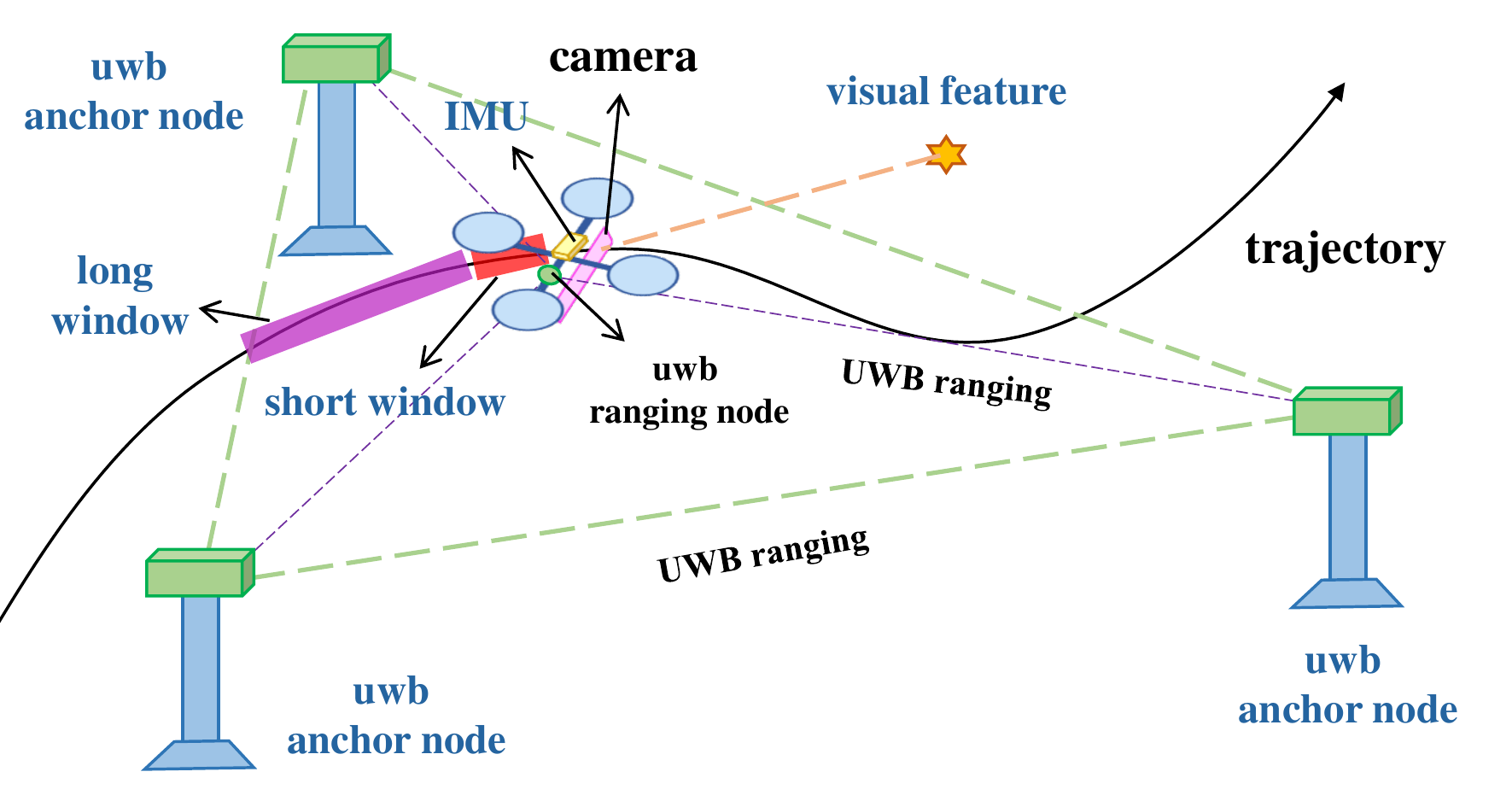}
  \captionsetup{font={footnotesize }}
  \caption{An intuitive illustration of our proposed FEJ-VIRO system. The proposed FEJ-VIRO is a consistent filter which estimates the robot pose and UWB anchors' positions simultaneously by fusing visual-inertial measurements as well as ranging measurements from UWB anchors. Dash lines represent sensor measurements, and state values are written with blue font.} 
  \vspace{-6mm}
  \label{fig:first-fig}
\end{figure}

In this paper, we follow the lightweight line of filter-based methods to first propose a multi-state constraint Kalman filter (MSCKF) \cite{MSCKF} based fusion solution to VIRO (as shown in Fig.\ref{fig:first-fig}), which is less computationally intensive than EKF based methods and satisfies the on-board computing requirements of robotic systems. In addition, we propose two techniques to solve the state modeling of UWB measurements in theory. One is for robust initialization of the UWB, the other is for consistent fusion of the UWB measurements into VIRO. For the initialization problem of UWB, the long-short sliding window strategy is proposed which solves the problem of initialization which allows to initialize the UWB anchors' positions with sufficient measurements as well as to evaluate the covariance for state augmentation. For the consistent fusion of UWB measurements, we first analyze the observability of the VIRO system with unknown UWB anchors' position and prove that there are \textit{four} unobservable directions, which correspond to the global translation and the global rotation about the gravity, despite the fusion of ranging sensors. Although the VIRO system can not achieve drift-free estimation, we can improve estimate accuracy by fusing the UWB ranging measurements. However, we prove that fusing the UWB measurements by direct EKF update operation causes inconsistency. Therefore, we propose a \textit{FEJ-VIRO} framework which linearizes the ranging measurement models at the first-estimate of state, which is known as the FEJ technique, to address the inconsistency issue. Experiments show that the reasonable initialization and consistent filter designing bring obvious improvements to estimation accuracy to \textit{FEJ-VIRO}. In summary, the contributions of this paper are listed as follows:

\begin{itemize}
\item A lightweight MSCKF-based Visual-inertial-ranging odometry (VIRO) method is proposed and a long-short sliding window strategy is presented for initializing UWB anchors and state augmentation.
\item The observability analysis of visual-inertial-ranging estimator is derived which points out \textit{four} unobservable directions of VIRO existing in the ideal case.
\item A FEJ-VIRO framework is proposed which extends the FEJ technique to maintain consistency, leading to obvious improvement in estimation accuracy.
\item Experiments on multiple simulation and real-world datasets validate the derived observability analysis and the effectiveness of the proposed FEJ-VIRO method.
\end{itemize}

\section{Related Work} \label{sec:related_work}
There are extensive works on visual-inertial odometry (VIO) \cite{OpenVINS}\cite{MSCKF}\cite{SVO2} and visual-inertial SLAM (VI-SLAM) \cite{ORB-SLAM3}\cite{qin2017vins}. The main difference between VIO and VI-SLAM is whether to use global information: VI-SLAM gains better accuracy by performing mapping (and thus loop closure), while the localization errors of VIO grows unbounded due to the lack of global information\cite{huang_review} \cite{jiao20202} \cite{chen2020deep}. However, it is well known that the computational complexity of VI-SLAM is significantly higher than VIO due to the iterative non-linear optimization, which makes it unsuitable to apply VI-SLAM to on-board devices\cite{MSCKF}\cite{OpenVINS}. As a result, there are two lines of methods to achieve highly accurate estimation with bounded computational complexity: by reducing the complexity of VI-SLAM \cite{keyframe_slam}\cite{isam}\cite{isam2} and by incorporating global information into VIO framework\cite{GPS_Huang}\cite{viral_fusion}. In this paper, we focus on the multi-sensor fusion methods to reduce localization drifts of VIO by incorporating ranging measurements from multiple unknown UWB anchors.

Over the last decades, many researchers have investigated in the use of UWB measurements for localization \cite{uwb_rough_init}\cite{viral_fusion}\cite{viral_slam}\cite{nguyen2021ntuviral}. Some early works\cite{kown_uwb1}\cite{known_uwb2}\cite{known_uwb3} assume prior knowledge of anchor positions and achieve drift-free estimation of the robots. However, these methods in practice require an offline calibration of the initial position of robots, which is difficult in dynamic and large-scale environments. In recent years, researchers propose to estimate state and UWB anchor positions simultaneously. \cite{uwb_rough_init} propose a EKF-based framework to initialize anchor positions and estimate them jointly with the robot position. However, the initialization accuracy of this method is not sufficiently high and the algorithm fail easily due to the ill-conditioned of matrix. \cite{single_uwb} propose to generate an initial guess of the UWB anchor position with a variant of the Levenberg-Marquardt method and simultaneously estimate the position with the robot state. \cite{viral_fusion} propose an estimator to fuse visual-inertial-ranging-lidar measurements where the UWB anchor positions are also assumed unknown and need to be estimated online, based on which \cite{viral_slam} further propose mapping and loop-closure with these measurements. Note that most recent methods \cite{single_uwb}\cite{viral_fusion}\cite{viral_slam} are based on non-linear optimization and relies on global bundle adjustment to achieve high accuracy, which is computational intensive and hard to run on-board in real-time. In this paper, we propose a filter-based visual-inertial-ranging system which incorporate multi-sensor measurements in an efficient MSCKF framework. We also assume unknown UWB anchor position and propose a method to initialize UWB anchor positions as well as its initial variance and covariance with the state.

The observability and consistency of VIO has been extensively studied. \cite{Consistency_ana} \cite{li_high} proved that visual-inertial estimator has four unobservable directions which correspond to the global translation and the global rotation about the gravity. They further study the inconsistency of EKF-based VIO and propose a first-estimated Jacobian technique to address the inconsistency issue. \cite{Consistency_ana} proved the key cause of inconsistency is the gain of spurious information along unobservable directions, which results in the over-confidence of the estimator. Extensive works studied this problem and proposed many techniques to address it \cite{zhang2022toward}. We refer the readers to \cite{past_review} for detailed review. In this paper, we prove the general VIRO system also has four unobservable directions despite of the fusion UWB measurements. Although we can not achieve drift-free estimation with VIRO, we can reduce the localization drifts of VIO by incorporating ranging measurements. However, from the unobservable directions of the VIRO, we find that fusing the ranging measurements with standard EKF-update results in inconsistency, which is caused by the estimator's over-confidence (as \cite{Consistency_ana}). Based on these analyses, we propose leveraging FEJ technique to address the inconsistency issue.

\begin{figure}
  \centering
  \includegraphics[width=0.8\columnwidth]{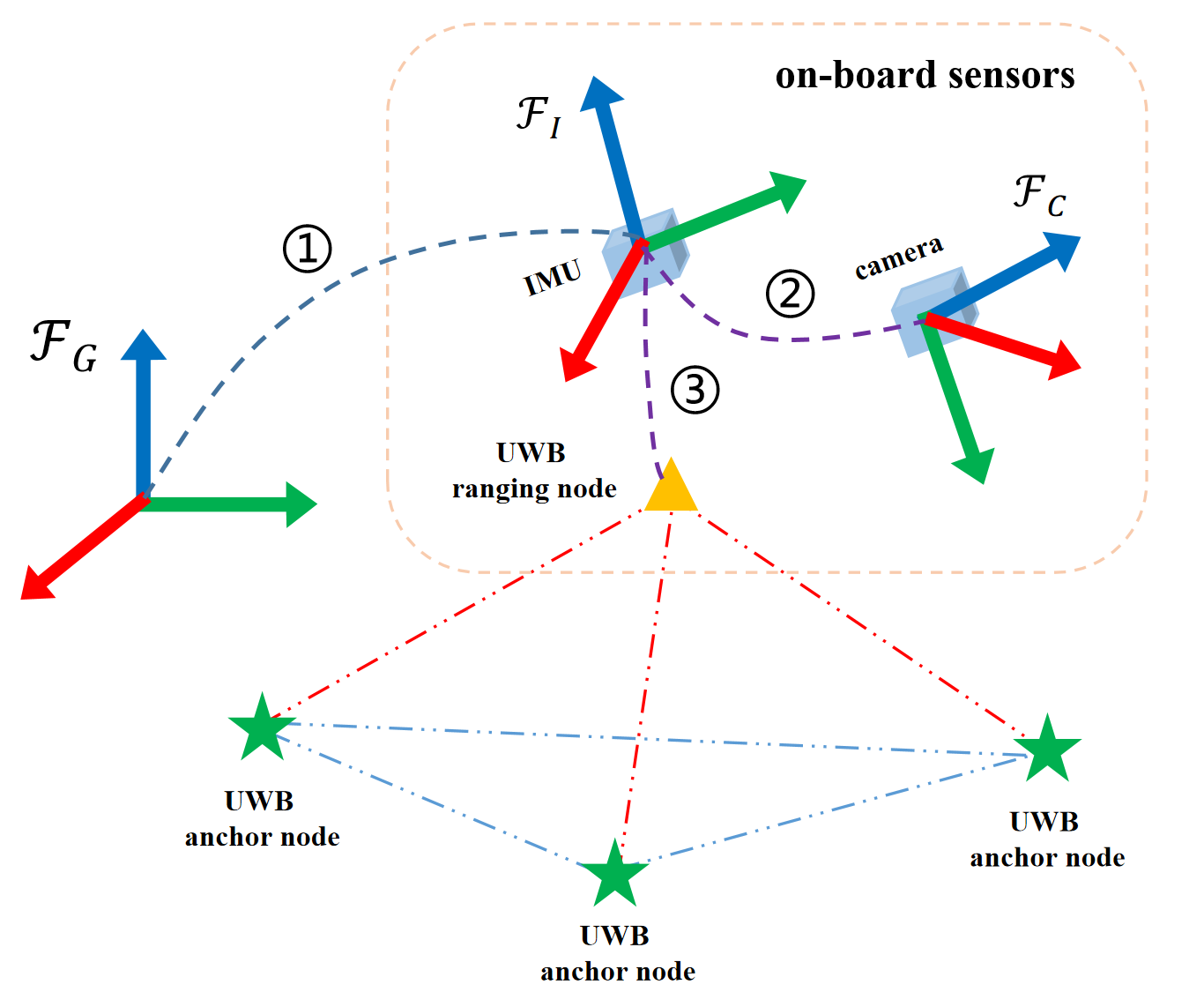}
  \captionsetup{font={footnotesize}}
  \caption{
    Illustration of the frames and sensors used in this work. Note that \textit{UWB ranging node} represents the UWB node equipped on the robot, and \textit{UWB anchor node} represents the UWB node that keep still in the global frame. The \textcircled{1} represents the robot pose at $t_k$; the \textcircled{2} is the extrinsic parameters between the IMU and camera; and \textcircled{3} represents the body-offset of the UWB ranging node. 
  }
  \label{fig:frame-sensors}
  \vspace{-6mm}
\end{figure}

\section{Estimator Design}
In this section, we describe in detail the filter based VIRO, which fuses measurements from UWB, cameras, and IMU in a MSCKF framework. We assume that there are three UWB anchors available in the environment, and we need to estimate their global positions simultaneously. An illustration of our system can be seen in \ref{fig:first-fig}, where the dash lines represents sensor measurements, and we write the states needs to be estimated with blue font.

Before describing our estimator, we first present some notations. There are three corrdinate frames in this paper, which are the \textit{global frame} $\mathcal F_G$, the \textit{IMU frame} $\mathcal F_I$ and the camera frame $\mathcal F_C$. Fig.\ref{fig:frame-sensors} illustrates the frames and sensors used in the VIRO. 

\subsection{State Vector}

The state vector of the proposed VIRO at time $t_k$ consists of the current IMU states, 
the position of SLAM features, 
the position of UWB anchors,  
and a short window and a long window which contain cloned IMU poses corresponding to the past images:
\begin{align} 
\label{eq:msckf-state} 
\mathbf x_k &= \begin{bmatrix} \mathbf x_{I,k}^\top & {^G \mathbf p_{f,k}^\top} & \mathbf x_{a,k}^\top & \mathbf{x}_{cl,k}^\top \!&\! \mathbf{x}_{\ell,k}^\top \end{bmatrix}^\top \\[3pt]
\label{eq:msckf-imu-state}
\mathbf x_{I,k} &= \begin{bmatrix} ^I_G\bar{\mathbf q}_k^\top & \mathbf b_{g,k}^\top & ^G\mathbf v_{I,k}^\top & \mathbf b_{a,k}^\top & ^G\mathbf p_{I,k}^\top \end{bmatrix}^\top \\[3pt]
\label{eq:msckf-anchor-state}
\mathbf x_{a,k} &= \begin{bmatrix}
  {^G \mathbf p_{a_1}^\top} \!&\! {^G \mathbf p_{a_2}^\top} \!&\! {^G \mathbf p_{a_3}^\top}
\end{bmatrix} ^\top \\[3pt]
\label{eq:msckf-clone-state}
\mathbf{x}_{cl,k} &= \scalemath{0.95} {
\begin{bmatrix}
{}^I_G{\bar{\mathbf q}_k}^{\top} \!&\! {}^G\mathbf{p}_{I,k}^{\top} \!&\! \cdots \!&\! {}^I_G{\bar{\mathbf q}_{k-m+1}}^{\top} \!&\! {}^G\mathbf{p}_{I,{k-m+1}}^{\top}
\end{bmatrix}^{\top} } \\[3pt]
\label{eq:msckf-long-window}
\mathbf{x}_{\ell,k} &= \scalemath{0.95} {
\begin{bmatrix}
{}^I_G{\bar{\mathbf q}_{\ell_1}}^{\top} \!&\! {}^G\mathbf{p}_{I,\ell_1}^{\top} \!&\! \cdots \!&\! {}^I_G{\bar{\mathbf q}_{\ell_M}}^{\top} \!&\! {}^G\mathbf{p}_{I,{\ell_M}}^{\top} 
\end{bmatrix}^{\top} } 
\end{align}
where $^G\mathbf v_I$ is the IMU velocity in the global frame; 
$\mathbf b_g$ and $\mathbf b_a$ denote the gyroscope and accelerator biases;
${^G \mathbf p_{f,k}}$ is the feature's position expressed in the global frame;
$^G \mathbf p_{a_i} (i=1,2,3)$ are the position of UWB anchors expressed in the global frame;
$\{{}^I_G{\bar{\mathbf q}_{k-i}},  {}^G\mathbf{p}_{I,{k-i}} \}$ ($i=0,\cdots,m-1$) are the cloned IMU poses at time $t_{k-i}$;
and $\{{}^I_G{\bar{\mathbf q}_{\ell_j}},  {}^G\mathbf{p}_{I,{\ell_j}} \}$ ($j=1,\cdots,M$) are the long window for UWB anchor's initialization, which will be introduced in Sec.\ref{sec:UWB-init}.

\subsection{IMU Propagation} \label{sec:imu-prop}

The state is propagated forward with the IMU linear acceleration and angular velocity:
\begin{align}
\bm\omega_m(t_k) &= {^I\bm\omega}(t_k) + \mathbf b_g(t_k) + \mathbf n_g(t_k) \label{eq:imu-wm}\\[3pt]
\mathbf a_m(t_k) &= \scalemath{0.95} {
^I_G\mathbf R(t_k)
\left( {^G\mathbf a}_I(t_k) + {^G\mathbf g}  \right) + \mathbf b_a(t_k) + \mathbf n_a(t_k) }
\label{eq:imu-am}
\end{align}
where $\bm\omega_m$ and $\mathbf a_m$ are the raw inertial measurement data, ${^G\mathbf g}$ is the gravitational acceleration expressed in $\mathcal F_G$, 
and $\mathbf n_g$ and $\mathbf n_a$ are zero-mean white Gaussian noise.
We propagate the state estimate and the covariance from time $t_k$ to $t_{k+1}$ based on the inertial kinematic model $\mathbf f (\cdot)$\cite{MSCKF}:
\begin{align}
\hat{\mathbf x}_{k+1|k} &= \mathbf f (\hat{\mathbf x}_{k|k}, \mathbf a_m(t_k:t_{k+1}), \bm \omega_m(t_k:t_{k+1}), \mathbf 0, \mathbf 0) \label{eq:imu-state-est-prop} \\[3pt] 
\mathbf P_{k+1|k} &= \bm\Phi(t_{k+1},t_k)  \mathbf P_{k|k} \bm \Phi(t_{k+1},t_k)^\top + \mathbf Q_{d,k} \label{eq:cov_propagate}
\end{align}
where the zero-mean noise vectors are represented by the last two $\mathbf 0$ entries of \eqref{eq:imu-state-est-prop}, and $\bm \Phi_k$ and $\mathbf Q$ are the state transition matrix and discrete noise covariance.

\subsection{Camera measurement model} \label{sec:cam-meas-model}
In this subsection, we review concisely the visual observation model with the calibrated perspective camera assumption.
Specifically, at time $t_k$, the position of a tracked feature point in the camera frame is $^{C_j}\mathbf p_{f} = [ x_j ~ y_j ~ z_j ]^\top$.
The feature measurement can be obtained by projecting its 3D pose onto the image plane with the projection model. 
For $j=k+1,\cdots,k-m$, a range-bear measurement model can be written as:
\begin{align} \label{eq:meas-model}
  \mathbf z_{n,j} & = \Pi({^{C_j} \mathbf p_f}) + \mathbf n_{z}
  = \frac{1}{z_j} \begin{bmatrix}
    x_j \\ y_j
  \end{bmatrix} + \mathbf n_z \\[3pt]
  {^{C_j}\mathbf p_{f}} &= 
  {^C_I\mathbf R} {^{I_j}_G\mathbf R} \left({^G\mathbf p_{f}} - {^G\mathbf p_{I,j}} \right) + {^C\mathbf p_I}
\label{eq:meas-eq}
\end{align}
where ${^G\mathbf p_{f}}$ is the 3D position of the tracked feature in $\mathcal F_G$; 
and $\mathbf n_{(r)}$ and $\mathbf n_{(b)}$ are the zero-mean white Gaussian measurement noise of range and bear measurements. We refer readers to \cite{MSCKF} for detailed presentation of visual update and null-space projection. 

\subsection{Ranging Update} \label{sec:ranging-update}
We leverage UWB ranging measurements to update the robot pose and the global position of UWB anchor nodes. Following \cite{nguyen2021ntuviral}, the UWB ranging measurement from the i-th UWB anchor node, $d_{r_i}$, can be modeled as:
\begin{align} \label{eq:uwb-ranging-model}
d_{r_i} &= h(\mathbf x_I, {^G \mathbf p_{a_i}}) \nonumber  \\[3pt]
 &= \Vert {^G \mathbf p_I} + {^I_G \mathbf R ^\top} {^I \mathbf p_r} - {^G \mathbf p_{a_i}} \Vert + d_{bias}
\end{align}
where $\{{^I_G \mathbf R}, {^G \mathbf p_I} \}$ is the IMU pose;
${^I \mathbf p_r}$ is the position of UWB ranging node expressed in the IMU frame, which can be calibrated offline; 
${^G \mathbf p_{a_i}}$ is the global position of the i-th UWB anchor node;
and $d_{bias}$ is the bias of distance measurements, which can be calibrated offline\cite{uwb_bias_calib}.

The distance measurements between UWB anchor nodes are available to the robot by robot-node communication. The distance measurement between the i-th UWB anchor and the j-th UWB anchor, $d_{e_{ij}}$, can be modeled as:
\begin{align} \label{eq:uwb-echo-model}
d_{e_{ij}} = \Vert {^G \mathbf p_{a_i}} - {^G \mathbf p_{a_j}} \Vert + d_{bias}
\end{align}
where ${^G \mathbf p_{a_i}}$ and ${^G \mathbf p_{a_j}}$ are the global position of the i-th UWB anchor node and the j-th UWB anchor node;
and $d_{bias}$ is the bias of distance measurements, which can be calibrated offline\cite{uwb_bias_calib}.

We can derive a concise expression of Jacobian by defining $d_{2r_i} = (d_{r_i}-d_{bias})^2 = \Vert {^G \mathbf p_I} + {^I_G \mathbf R ^\top} {^I \mathbf p_r} - {^G \mathbf p_{a_i}} \Vert^2$. Then we perturb $d_{2r_i}$ to derive the measurement Jacobians:
\begin{align} \label{eq:ranging-jacobian}
    \frac{\partial \tilde{d}_{2r_i}}{\partial {^G \tilde{\mathbf p}_I}} &= 2\cdot({^G \hat{\mathbf p}_r} - {^G \hat{\mathbf p}_{a_i}})^\top \\
    \frac{\partial \tilde{d}_{2r_i}}{\partial {^I_G \tilde{\bm \theta}}} &= -2\cdot({^G \hat{\mathbf p}_{a_i}} - {^G \hat{\mathbf p}_u})^\top\cdot {^I_G \hat{\mathbf R}^\top} \lfloor {^I \mathbf p_r} \times \rfloor \\
    \frac{\partial \tilde{d}_{2r_i}}{\partial {^G \tilde{\mathbf p}_{a_i}}} &= 2\cdot({^G \hat{\mathbf p}_{a_i}} - {^G \hat{\mathbf p}_r})^\top
\end{align}
where ${^G \hat{\mathbf p}_r} = {^G \mathbf p_I} + {^I_G \mathbf R ^\top} {^I \mathbf p_r}$ is the position of UWB ranging node represented in the global frame; and ${^I \mathbf p_r}$ is calibrated offline.

Also in the same way, by defining $d_{2e_{ij}} = (d_{e_{ij}} - d_{bias})^2$, we have:
\begin{align} \label{eq:echo-jacobian}
    \frac{\partial \tilde{d}_{2e_{ij}}}{\partial {^G \tilde{\mathbf p}_{a_i}}} &= 2\cdot({^G \hat{\mathbf p}_{a_i}} - {^G \hat{\mathbf p}_{a_j}})^\top \\
    \frac{\partial \tilde{d}_{2e_{ij}}}{\partial {^G \tilde{\mathbf p}_{a_j}}} &= 2\cdot({^G \hat{\mathbf p}_{a_j}} - {^G \hat{\mathbf p}_{a_i}})^\top    
\end{align}
Now we can perform EKF update with the square-unbiased ranging measurements.

\subsection{Data synchronization}
In this work, we interpolate the ranging measurements to synchronize the visual and ranging measurements. 
Assume that we receive an image (or stereo images) at times $t_k$, and the adjacent ranging measurements are received at $t_\alpha$ and $t_\beta$, where $t_\alpha < t_k < t_\beta$, the aligned ranging measurement is computed by:
\begin{align} \label{eq:ranging-interpolate}
d_{r,k} = d_{r, \alpha} + \frac{t_k - t_\alpha}{t_\beta - t_\alpha} (d_{r, \beta} - d_{r, \alpha})
\end{align}
where $d_{r, \alpha}$ and $d_{r, \beta}$ are ranging measurements received at $t_\alpha$ and $t_\beta$. Note that both $d_{r, \alpha}$ and $d_{r, \beta}$ should be from the same UWB anchor node. In practice, if the time offsets $t_k - t_\alpha$ and $t_\beta - t_k$ are bigger than a threshold, we discard this pair of ranging measurements.

\section{UWB Initialization} \label{sec:UWB-init}

We design a long-short-window structure to estimate the global position as well as the initial covariance. The short window is the same as the MSCKF sliding window, while the long window contains a list of key-frames selected according to the distance interval. Once the length of the long window is over a threshold, we leverage the poses in the long window and corresponding ranging measurements to estimate all UWB anchors' global positions.

Assume that there are $n$ poses in the long-short window, we build the linear problem following \cite{uwb_rough_init}:
\begin{align} \label{eq:uwb-linear-init}
\scalemath{0.85} {
\underbrace{
\begin{bmatrix}
-2 x_r(t_1) & -2 y_r(t_1) & -2 z_r(t_1) & 1 \\
-2 x_r(t_2) & -2 y_r(t_2) & -2 z_r(t_2) & 1 \\
\vdots      & \vdots      & \vdots      & 1 \\
-2 x_r(t_n) & -2 y_r(t_n) & -2 z_r(t_n) & 1 
\end{bmatrix}
}_{\mathbf A} \cdot
\underbrace{
\begin{bmatrix}
x_{a_i} \\ y_{a_i} \\ z_{a_i} \\ D_{a_i}
\end{bmatrix} 
}_{\mathbf x}
= \underbrace{
\begin{bmatrix}
\Delta(t_1) \\ \Delta(t_2) \\ \vdots \\ \Delta(t_n)
\end{bmatrix}
}_{\mathbf b}
} 
\end{align}
with
\begin{align}
D_{a_i} &= x_{a_i}^2 + y_{a_i}^2 + z_{a_i}^2 \\[3pt]
\Delta  &= (d_{r_i} - d_{bias})^2 - (x_r^2 + y_r^2 + z_r^2)
\end{align}
where ${^G \mathbf p_r} = {^G \mathbf p_I} + {^I_G \mathbf R^\top}\cdot {^I \mathbf p_r} = \begin{bmatrix} x_r & y_r & z_r\end{bmatrix}^\top$ is the global position of the UWB ranging node;
${^G \mathbf p_{a_i}} = \begin{bmatrix}x_{a_i} & y_{a_i} & z_{a_i}\end{bmatrix}^\top$ is the i-th anchor's global position;
$d_{r_i}$ is the raw measurement from the i-th UWB anchor node;
and $d_{bias}$ is the bias of ranging measurements.
By solving \eqref{eq:uwb-linear-init}, we can roughly estimate the global position of the i-th UWB anchor node.

Since the matrix $\mathbf A$ in \eqref{eq:uwb-linear-init} is usually ill-conditioned, we jointly optimize all anchor nodes' global position by minimizing the following cost function:
\begin{align} \label{eq:uwb-nonlinear-init}
\sum_{^G \mathbf p_{a_i}} \sum_{k=1 \cdots n} \frac{(d_{r_i,k} - d_{bias}-\Vert {^G \mathbf p_{a_i}} - {^G \mathbf p_{r,k}} \Vert)^2}{\sigma_r^2} + \ \ \ \  & \nonumber \\[3pt]
\sum_{\scalemath{.7}{{^G \mathbf p_{a_i}}, {^G \mathbf p_{a_j}}}}
\sum_{d_e \in \Omega_{ij}} 
\frac{(d_e - d_{bias} - \Vert {^G \mathbf p_{a_i}} - {^G \mathbf p_{a_j}} \Vert)^2}{\sigma_e^2} &
\end{align}
where $d_{r_i,k}$ is the ranging measurement from the i-th UWB anchor node at time-step $k$;
${^G \mathbf p_{r,k}} = {^G \mathbf p_{I,k}} + {^I_G \mathbf R^\top} \cdot {^I \mathbf p_r}$ is the global position of the UWB ranging node;
${\sigma_r^2}$ is the noise density of ranging measurements;
$d_e$ is the ranging measurements between two UWB anchor nodes;
$\Omega_{ij}$ is the set of measurements between the i-th and the j-th UWB anchor nodes during the last $n$ time-step;
$\sigma_e^2$ is the noise density of ranging measurements between UWB anchor nodes.

\begin{figure}
  \centering 
    \includegraphics[width=0.85\columnwidth,trim=50 180 380 100,clip]{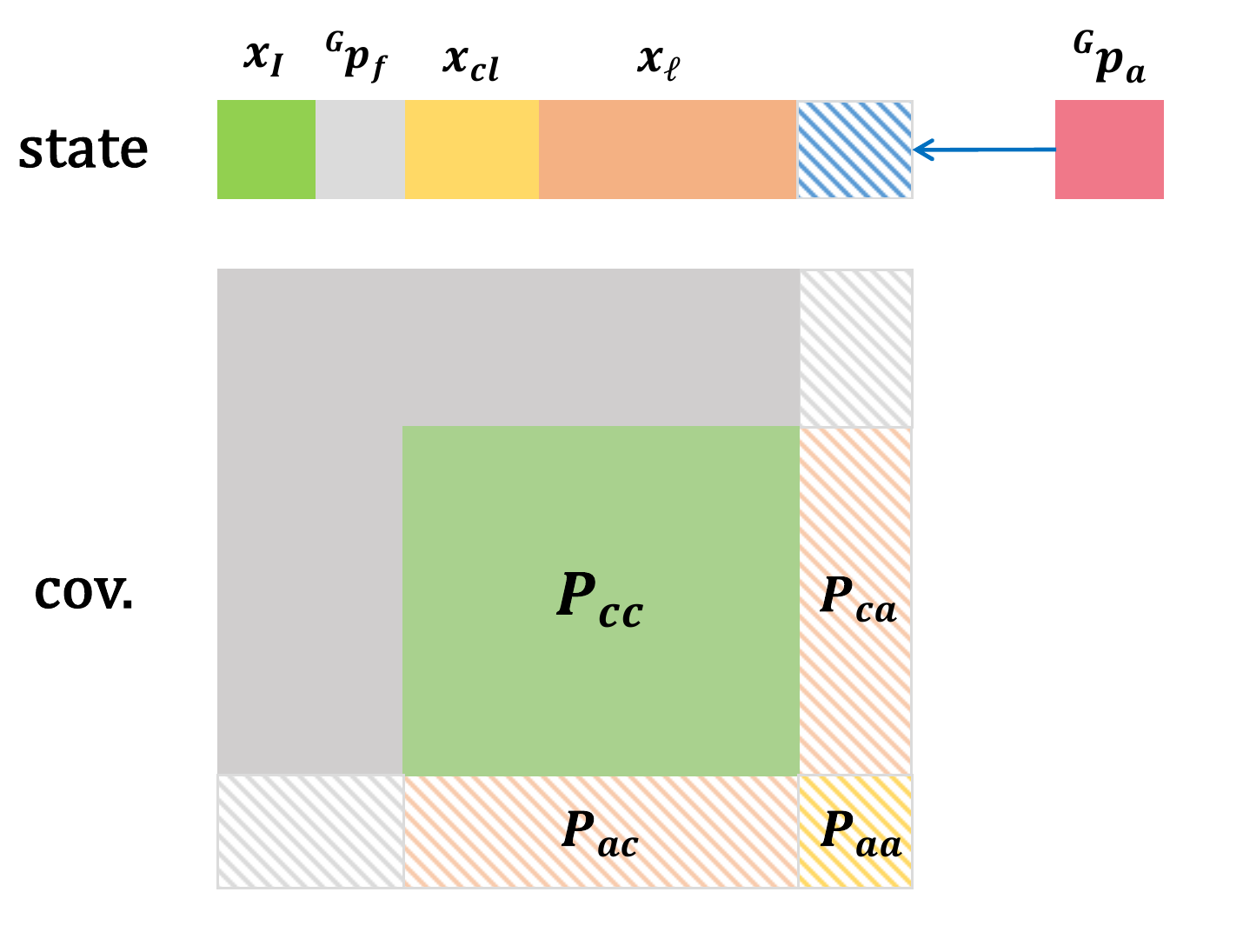}
  \captionsetup{font={footnotesize }}
  \caption{
    Illustration of the state augmentation step after the UWB anchors' initialization. We evaluate the covariance matrix blocks related to the UWB anchor positions and augment the state vector and covariance matrix after that.
  }
  \label{fig:state-augment}
  \vspace{-6mm}
\end{figure}

By solving \eqref{eq:uwb-nonlinear-init} with Gaussian-Newton method, we can obtain refined global positions of all UWB anchor nodes, $^G \hat{\mathbf p}_{a_i} (i=1,2,3)$ . The last step of our UWB initialization is to estimate the initial covariance matrix. In this work, we estimate the covariance of the i-th anchor node with all of the $n$ ranging measurements: $\mathbf d_{r_i} = \begin{bmatrix} d_{r_i,1} & d_{r_i,2} & \cdots & d_{r_i,n} \end{bmatrix}^\top$, where $\scalemath{.9}{d_{r_i,k} = h_k(\mathbf x_{I,k}, {^G \mathbf p_{a_i}}) + n_{i,k}}$ is the ranging measurement from the i-th anchor node at the k-th time-step.
Following \cite{li2014visual}, we linearize \eqref{eq:uwb-ranging-model}:
\begin{align} \label{eq:long-window-residual}
\mathbf r_i &= \mathbf d_{r_i} - \mathbf h(\hat{\mathbf x}_I, {^G \hat{\mathbf p}_{a_i}}) \\[3pt]
&\simeq \mathbf H_x \tilde{\mathbf x}_I + \mathbf H_a {^G \tilde{\mathbf p}_{a_i}} + \mathbf n_i
\end{align}
Since the number of ranging measurements is larger than the dimension of  ${^G \tilde{\mathbf p}_{a_i}}$, we perform Givens-rotation to zero-out the excess rows in $\mathbf H_a$:
\begin{align} \label{eq:long-window-Jacobian}
    \begin{bmatrix}
    \mathbf r_{i1} \\ \mathbf r_{i2}
    \end{bmatrix} = \begin{bmatrix}
    \mathbf H_{x1} \\ \mathbf H_{x2}
    \end{bmatrix} \tilde{\mathbf x}_I +
    \begin{bmatrix}
    \mathbf H_{a1} \\ \mathbf 0
    \end{bmatrix} {^G \tilde{\mathbf p}_{a_i}} +
    \begin{bmatrix}
    \mathbf n_{i1} \\ \mathbf n_{i2}
    \end{bmatrix}
\end{align}
As a result, we evaluate the covariance matrix by:
\begin{align} 
    \label{eq:long-window-paa} \mathbf P_{aa} &= \mathbf H_{a1}^{-1} (\mathbf H_{x1} \mathbf P_{xx} \mathbf H_{x1}^\top + \sigma^2_r \mathbf I) \mathbf H_{a1}^{-\top} \\[3pt]
    \label{eq:long-window-pxa} \mathbf P_{xa} &= -\mathbf P_{xx} \mathbf H_{x1}^\top \mathbf H_{a1}^{-\top}
\end{align}
where $\sigma_r$ is the noise density of ranging measurements. Finally we augment the state and the covariance matrix.

\section{Consistency Analysis}
Without loss of generality, we study the observability based on a general EKF state vector with one visual feature and two UWB anchor nodes:
\begin{align}
    \scalemath{.9} {
    \mathbf x = \begin{bmatrix}
    {^I_G \bar{\mathbf q}^\top} & \mathbf b_g^\top & {^G \mathbf v_I^\top} & \mathbf b_a^\top & {^G \mathbf p_I^\top} & {^G \mathbf p_f^\top} & {^G \mathbf p_{a_1}^\top} & {^G \mathbf p_{a_2}^\top}
    \end{bmatrix}^\top }
\end{align}
where $\{{^I_G \bar{\mathbf q}}, {^G \mathbf p_I}\}$ is the IMU pose in the global frame; ${^G \mathbf v_I}$ is the IMU velocity; $\mathbf b_g$ and $\mathbf b_a$ are the IMU bias; ${^G \mathbf f}$ is the visual feature's global position; and ${^G \mathbf p_{a_i}}, (i=1,2)$ is the global position of UWB anchor node. Note that by considering two UWB anchor nodes, we can take ranging measurements between anchors into our analysis, and this setting can be easily generalized to the cases with multiple UWB anchors.

\subsection{Estimator Observability} \label{sec:observability-study}

The observability matrix has the following form:
\begin{align}
    \mathcal O = \begin{bmatrix}
    \mathbf H_1 \\
    \mathbf H_2 \bm \Phi_{2,1} \\
    \vdots \\
    \mathbf H_n \bm \Phi_{n,1}
    \end{bmatrix}
\end{align}
where $\mathbf H_k\ (k=1\cdots n)$ is the measurement Jacobian to update the state; and $\bm \Phi_{k,1}\ ( k=1\cdots n)$ is the state transition matrix described in Sec.\ref{sec:imu-prop}. Since we update the state with measurements from cameras and UWB anchors, the k-th block row of $\mathcal O$ can be written as:
\begin{align} 
    \mathbf H_k \bm \Phi_{k,1} = \begin{bmatrix}
    \mathbf H_{v,k} \bm \Phi_{k,1} \\
    \mathbf H_{u,k} \bm \Phi_{k,1}
    \end{bmatrix}
\end{align}
where the visual part $\mathbf H_{v,k} \bm \Phi_{k,1}$ has been studied in previous work\cite{huang_analysis}\cite{Consistency_ana}. In this paper we study the form and null space of the ranging part $\mathbf H_{u,k} \bm \Phi_{k,1}$.

During the IMU propagation described in Sec.\ref{sec:imu-prop}, the state transition matrix from the starting time-step $t_1$ to $t_k$ has the following form:
\begin{align} \label{eq:state-transition}
\bm \Phi_{k,1} &= \begin{bmatrix}
    \bm \Phi_{I_K, I_1} & \mathbf 0_{15\times 9} \\
    \mathbf 0_{9\times 15} & \mathbf I_{9\times 9}
\end{bmatrix} \\[3pt]
\bm \Phi_{I_k, I_1} &= \begin{bmatrix}
\bm \phi_{11} & \bm \phi_{12} & \mathbf 0_3   & \mathbf 0_3   & \mathbf 0_3 \\
\mathbf 0_3   & \mathbf I_3   & \mathbf 0_3   & \mathbf 0_3   & \mathbf 0_3 \\
\bm \phi_{31} & \bm \phi_{32} & \mathbf I_3   & \bm \phi_{34} & \mathbf 0_3 \\
\mathbf 0_3   & \mathbf 0_3   & \mathbf 0_3   & \mathbf I_3   & \mathbf 0_3 \\
\bm \phi_{51} & \bm \phi_{52} & \bm \phi_{53} & \bm \phi_{54} & \mathbf I_3
\end{bmatrix}
\end{align}
where the analytical form of $\bm \Phi_{I_K, I_1}$ can be found in \cite{Consistency_ana}.
From Sec.\ref{sec:ranging-update}, the Jacobian matrix of UWB ranging measurements is:
\begin{align} \label{eq:uwb-jacobian-matrix}
\mathbf H_{u,k} = \scalemath{.8} {\begin{bmatrix}
\bm \zeta_{a_1} \cdot {^{I_k}_G \mathbf R^\top} \lfloor {^I \mathbf p_r} \times \rfloor & \mathbf 0_{3\times 9} & - \bm \zeta_{a_1} & \mathbf 0 & \bm \zeta_{a_1} & \mathbf 0 \\
\bm \zeta_{a_2} \cdot {^{I_k}_G \mathbf R^\top} \lfloor {^I \mathbf p_r} \times \rfloor & \mathbf 0_{3\times 9} & - \bm \zeta_{a_2} & \mathbf 0 & \mathbf 0 & \bm \zeta_{a_2} \\
\mathbf 0 & \mathbf 0_{3\times 9} & \mathbf 0 & \mathbf 0 & \bm \Lambda_a & - \bm \Lambda_a 
\end{bmatrix}
}
\end{align}
$\mathbf H_{u,k}$ is a $3\times 24$ matrix, and 
\begin{align}
    \bm \zeta_{a_i} &= 2\cdot ({^G \mathbf p_{a_i}} - {^G \mathbf p_{I,k}} - {^{I_k}_G \mathbf R^\top} \cdot {^I \mathbf p_u})^\top,\ \ i=1,2 \\[3pt]
    \bm \Lambda_a &= 2\cdot ({^G \mathbf p_{a_1}} - {^G \mathbf p_{a_2}})^\top
\end{align}
Note that the first two block rows of \eqref{eq:uwb-jacobian-matrix} corresponds to the ranging measurements between UWB anchor nodes and the UWB ranging node, while the last block row corresponds to the ranging measurements between two UWB anchor nodes. 
By combining \eqref{eq:state-transition} and \eqref{eq:uwb-jacobian-matrix}, we have:
\begin{align} \label{eq:UWB-obs-matrix}
&\mathbf H_{k,u} \bm \Phi_{k,1} = \nonumber \\[3pt]
& \scalemath{.8} {
\begin{bmatrix}
\bm \Gamma_{11} & \bm \Gamma_{12} & -\bm \zeta_{a_1} \bm \phi_{53} & -\bm \zeta_{a_1} \bm \phi_{54} & -\bm \zeta_{a_1} & \mathbf 0 & \bm \zeta_{a_1} & \mathbf 0 \\
\bm \Gamma_{21} & \bm \Gamma_{22} & -\bm \zeta_{a_1} \bm \phi_{53} & -\bm \zeta_{a_2} \bm \phi_{54} & -\bm \zeta_{a_2} & \mathbf 0 & \mathbf 0 & \bm \zeta_{a_1} \\
\mathbf 0 & \mathbf 0 & \mathbf 0 & \mathbf 0 & \mathbf 0 & \mathbf 0 & \bm \Lambda_a & -\bm \Lambda_a
\end{bmatrix}
}
\end{align}
which is a $3\times 24$ matrix, and 
\begin{align}
\bm \Gamma_{i1} &= \bm \zeta_{a_i}\cdot {^{I_k}_G \mathbf R^\top} \lfloor {^I \mathbf p_r} \times \rfloor \bm \phi_{11} - \bm \zeta_{a_i} \bm \phi_{51}, \ \ i=1,2 \\[3pt]
\bm \Gamma_{i2} &= \bm \zeta_{a_i}\cdot {^{I_k}_G \mathbf R^\top} \lfloor {^I \mathbf p_r} \times \rfloor \bm \phi_{12} - \bm \zeta_{a_i} \bm \phi_{52}, \ \ i=1,2 
\end{align}

\begin{figure*}[ht] 
  \centering 
  \subfloat[] {
    \includegraphics[width=0.65\columnwidth,trim=110 270 120 280,clip]{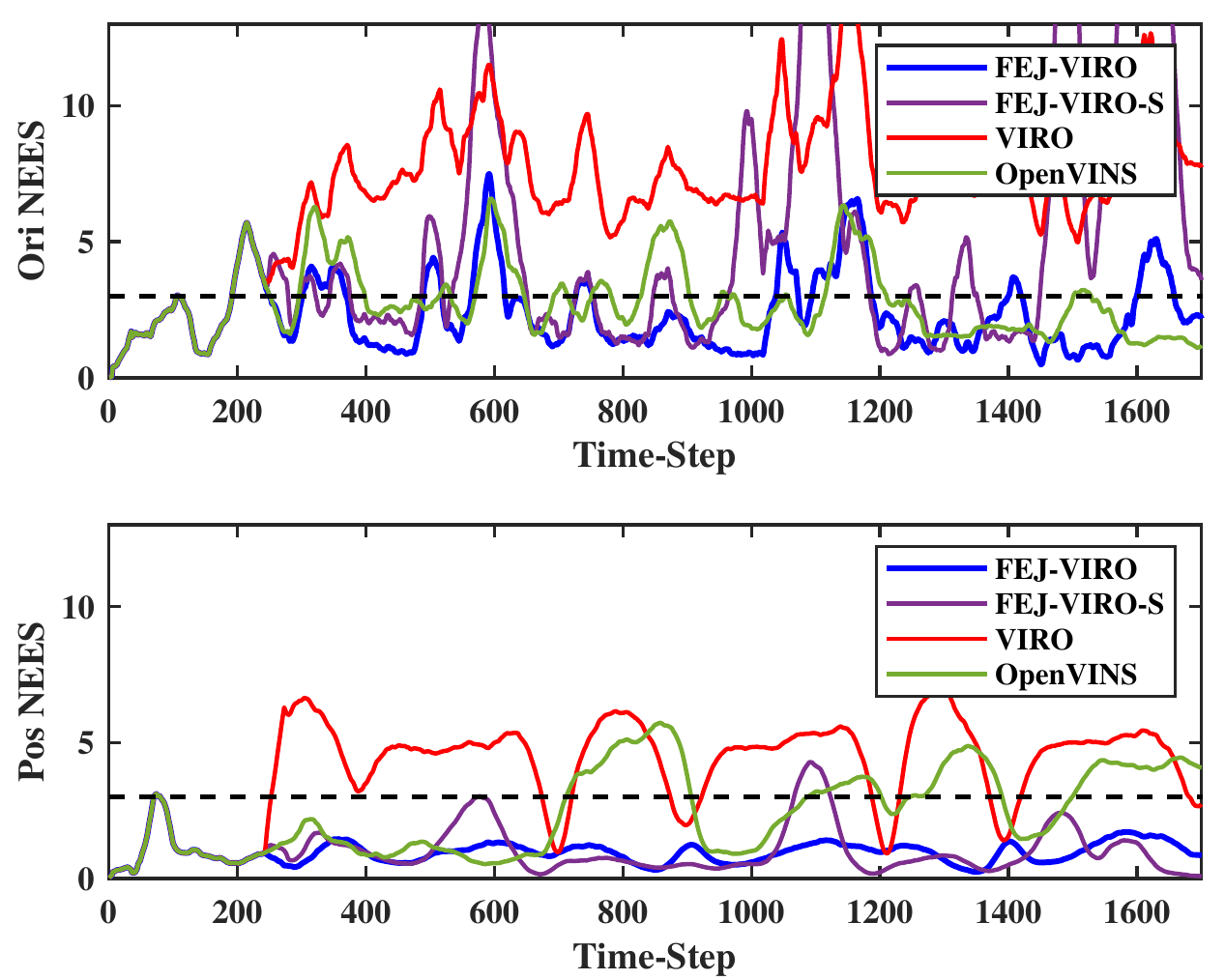}
  } 
  \subfloat[] {
    \includegraphics[width=0.65\columnwidth,trim=110 270 120 280,clip]{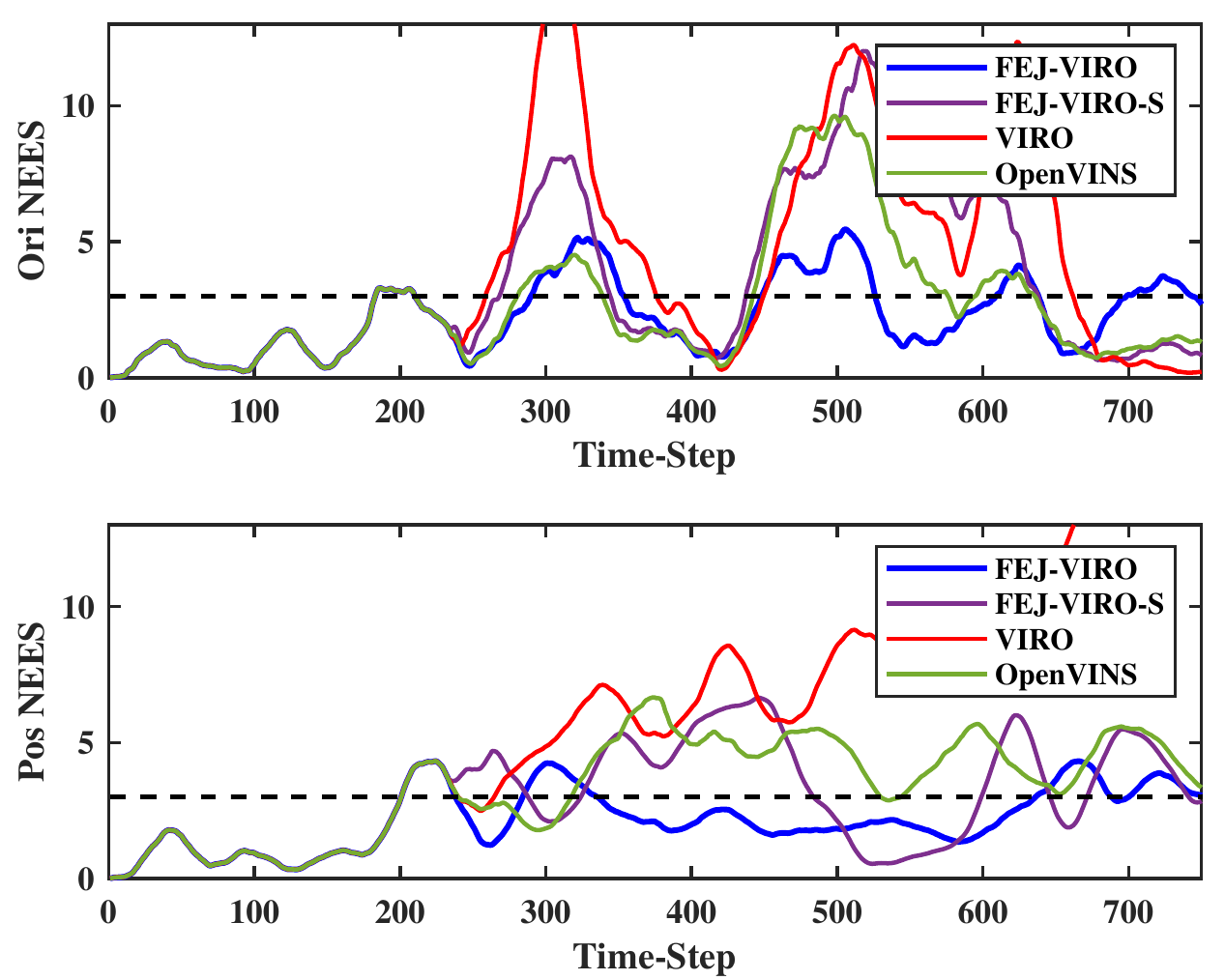}
  } 
  \subfloat[] {
    \includegraphics[width=0.65\columnwidth,trim=110 270 120 280,clip]{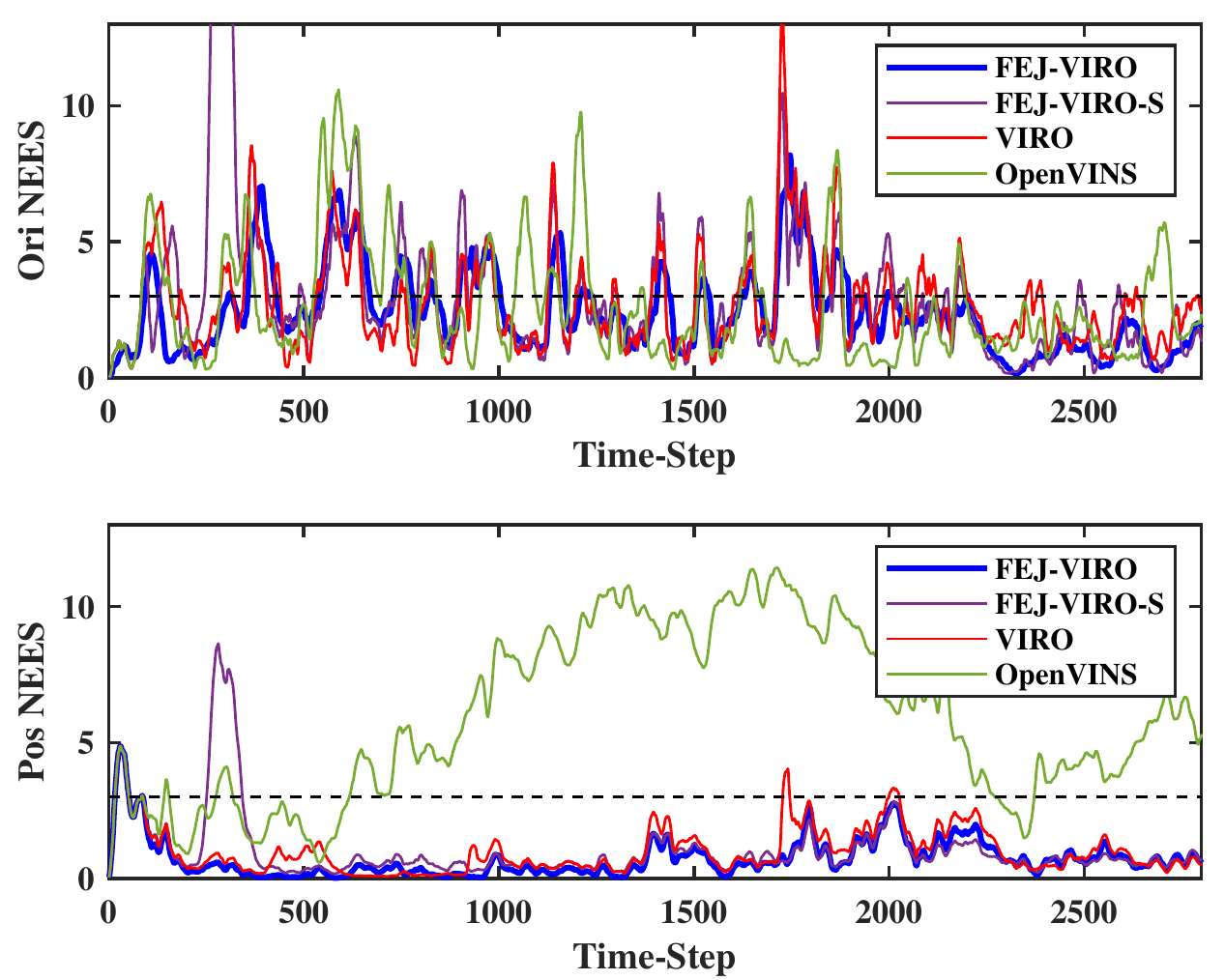}
  } 
  \captionsetup{font={footnotesize }}
  \caption{Normalized estimation error squared (NEES) results of \textit{OpenVINS}, \textit{VIRO}, \textit{FEJ-VIRO-S} and \textit{FEJ-VIRO} tested on three different simulation datasets. Where the \textit{FEJ-VIRO-S} is the \textit{FEJ-VIRO} without long-window. We enabled FEJ for \textit{OpenVINS} and for the visual-inertial part of \textit{VIRO}, so the inconsistency of \textit{VIRO} only results from state update with UWB ranging measurements. } 
  \vspace{-3mm}
  \label{fig:nees}
\end{figure*}


By study the analytical form of the $\mathcal O$, we have the following \textit{theorem}:
\begin{theorem} \label{thm:nullspace}
(Ideal observability proprieties of VIRO) The right null-space $\mathbf N_o$ of the observability matrix $\mathcal O$ of the linearized visual-inertial-ranging estimator
\begin{align}
    \mathcal O \cdot \mathbf N_o = \mathbf 0
\end{align}
is spanned by the following four directions:
\begin{align} \label{eq:null-space}
\mathbf N_o &=
    \begin{bmatrix}
    \mathbf 0_3 &  {^{I_1}_G \mathbf R} \cdot {^G \mathbf g} \\
    \mathbf 0_3 & \mathbf 0_3 \\
    \mathbf 0_3 & -\lfloor {^G \mathbf v_{I_1}} \times \rfloor \cdot {^G \mathbf g}\\
    \mathbf 0_3 & \mathbf 0_3 \\
    \mathbf I_3 & -\lfloor {^G \mathbf p_{I_1}}\times \rfloor \cdot {^G \mathbf g} \\
    \mathbf I_3 & -\lfloor {^G \mathbf p_f}\times \rfloor \cdot {^G \mathbf g} \\
    \mathbf I_3 & -\lfloor {^G \mathbf p_{a_1}} \times \rfloor {^G \mathbf g} \\
    \mathbf I_3 & -\lfloor {^G \mathbf p_{a_2}} \times \rfloor {^G \mathbf g}
    \end{bmatrix} \nonumber \\[3pt]
&= \begin{bmatrix} \mathbf N_1 & \mathbf N_2 \end{bmatrix}
\end{align}
\end{theorem}
\begin{proof}
See \textbf{Appendix}.
\end{proof}

\begin{theorem} \label{thm:nullspace_real}
(Actual observability proprieties of VIRO) For the actual case where the ranging measurement Jacobians are evaluated at the estimated states. The right null-space $\mathbf N_o$ of the observability matrix $\mathcal O$ of the linearized visual-inertial-ranging estimator is spanned by the following three directions:
\begin{align} \label{eq:null-space}
\mathbf N_o &=
\begin{bmatrix}
    \mathbf 0_3 & \mathbf 0_3 & \mathbf 0_3 & \mathbf 0_3 & \mathbf I_3 & \mathbf I_3 & \mathbf I_3 & \mathbf I_3 
\end{bmatrix} ^\top
\end{align}
\end{theorem}

\begin{proof}
See \textbf{Appendix}.
\end{proof}

From \textit{theorem}.\ref{thm:nullspace}, we can not achieve drift-free state estimate with visual-inertial-ranging estimator, because the global translation and yaw rotation are unobservable. But we can improve the estimation accuracy by multi-sensor fusion algorithms. 
The \textit{theorem}.\ref{thm:nullspace_real} states that if we linearize the ranging measurement models at the latest estimated state, the VIRO system gains spurious information about the global yaw direction, which results in the inconsistency of the estimator.

\subsection{Consistent Filter Design} \label{subsec:consistent-issue}
To design consistent state estimator, we need to force the unobservable subspace of the estimated model (the actual case) to be the same as that of the ideal case, hence preventing spurious information gain and reducing inconsistency. The inconsistency issue also happens in the EKF-based VIO, and several approaches has been proposed to design consistent filter, such as the first-estimate Jacobian (FEJ) idea\cite{FEJ_Huang} and the OC methodology\cite{OC_Huang}. 

In this paper, we extend the traditional FEJ to the VIRO system by always evaluating the measurement Jacobian matrices at the first-estimated robot pose and UWB anchors' positions. And we follow \cite{FEJ_Huang} to perform FEJ for the IMU propagation and visual update. With these changes, the null space of the observability matrix $\mathcal O$ in the actual case keeps the same as that in the ideal case. 

\section{Experiments}
We implement the proposed consistent visual-inertial-ranging filter based on \textit{OpenVINS}\cite{OpenVINS}, which is the state-of-the-art filter based visual-inertial estimator. We incorporate UWB ranging measurements in it and also implemented the UWB initialization described in Sec.\ref{sec:UWB-init}. Note that the non-linear optimization problem \eqref{eq:uwb-nonlinear-init} is solved by \textit{Ceres-Solver}\cite{ceres-solver}, which is an open source C++ library for modeling and solving large, complicated optimization problems. We evaluate the accuracy and consistency with both simulation and real-world experiments. 

\subsection{Simulation Experiments}
The main purpose of simulation measurements is to validate the inconsistency issue described in \ref{subsec:consistent-issue} and further evaluate the consistency of the proposed VIRO. We expanded the simulation utility provided by \textit{Open-VINS} to additionally simulate ranging measurements from three UWB anchor nodes with noise density of \textit{0.15m} and bias of \textit{-0.75m}. In order to validate the inconsistency issue described in \ref{subsec:consistent-issue}, we tested the proposed VIRO on three different simulation datasets, and evaluate the normalized estimation error squared (NEES, see p31 of \cite{model_based_estimation_note}) of them. Some key parameters used in our simulation are listed in Tab. \ref{tab:sim_params}. 

\begin{table}[ht]
\centering
\caption{Simulation Parameters}\label{tab:sim_params}
\renewcommand\arraystretch{1.3}{
\begin{tabular}{cccc}
\toprule
\textbf{Parameter}  & \textbf{Value} & \textbf{Parameter} & \textbf{Value} \\
\hline
Cam Freq.(Hz)     & 10       & IMU Freq.(Hz) & 200 \\
UWB Freq.(Hz)     & 60       & Num. UWB Anchor & 3 \\
Max. Feats        & 180      & Num. Clones   & 11 \\
Gyro. White Noise & 1.7e-04  & Gyro. Rand. Walk & 2.0e-05 \\
Accel. White Noise & 2.0e-03 & Accel. Rand. Walk & 3.0e-03 \\
UWB White Noise   & 1.5e-1 & UWB ranging bias(m) & -0.75 \\
\toprule
\end{tabular} }
\end{table}

\begin{figure*}
\centering
\hspace{0.02\columnwidth}
\subfloat[eee\_01] {
  \includegraphics[width=0.6\columnwidth]{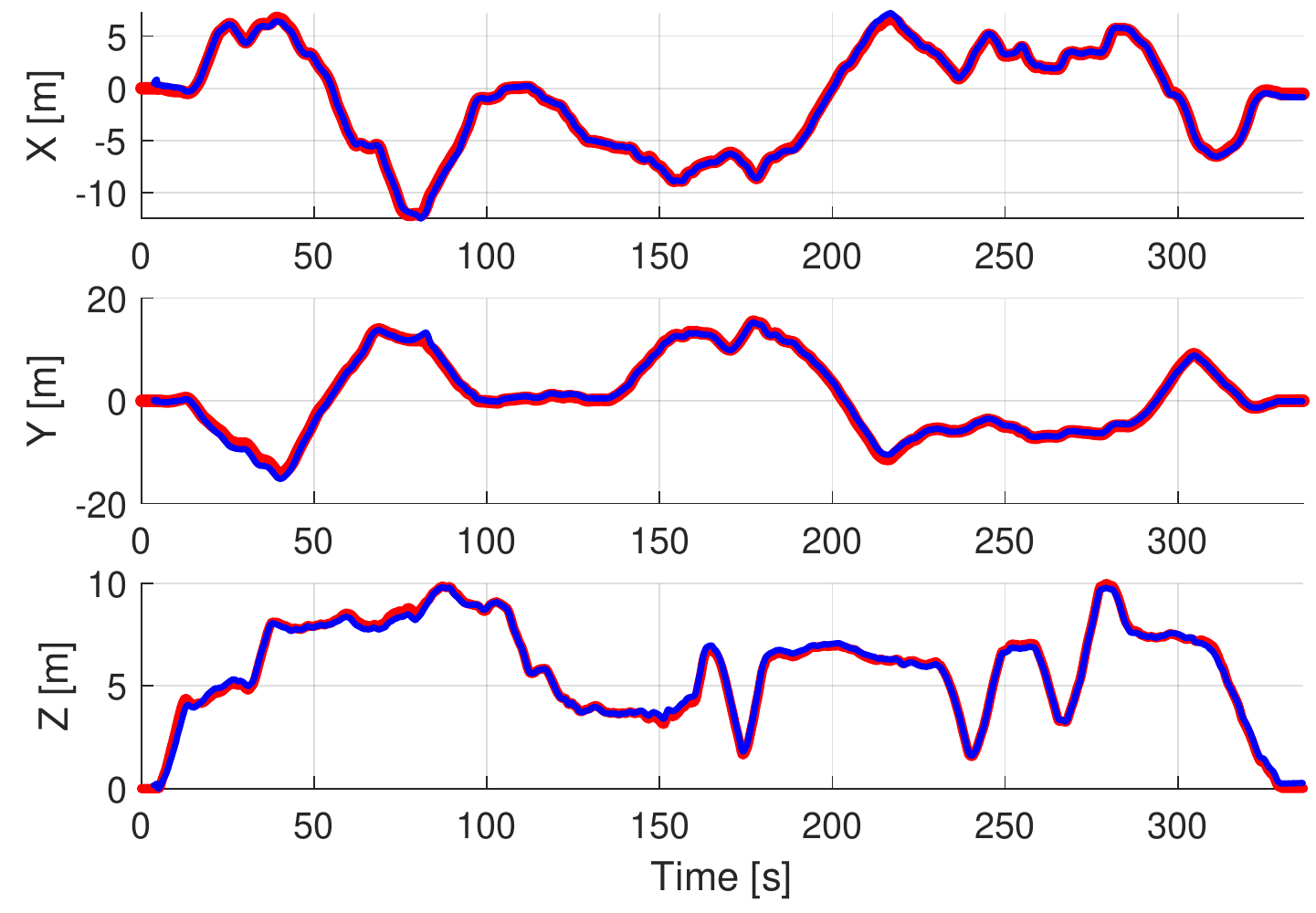}
} \vspace{3pt}
\subfloat[eee\_02] {
  \includegraphics[width=0.6\columnwidth]{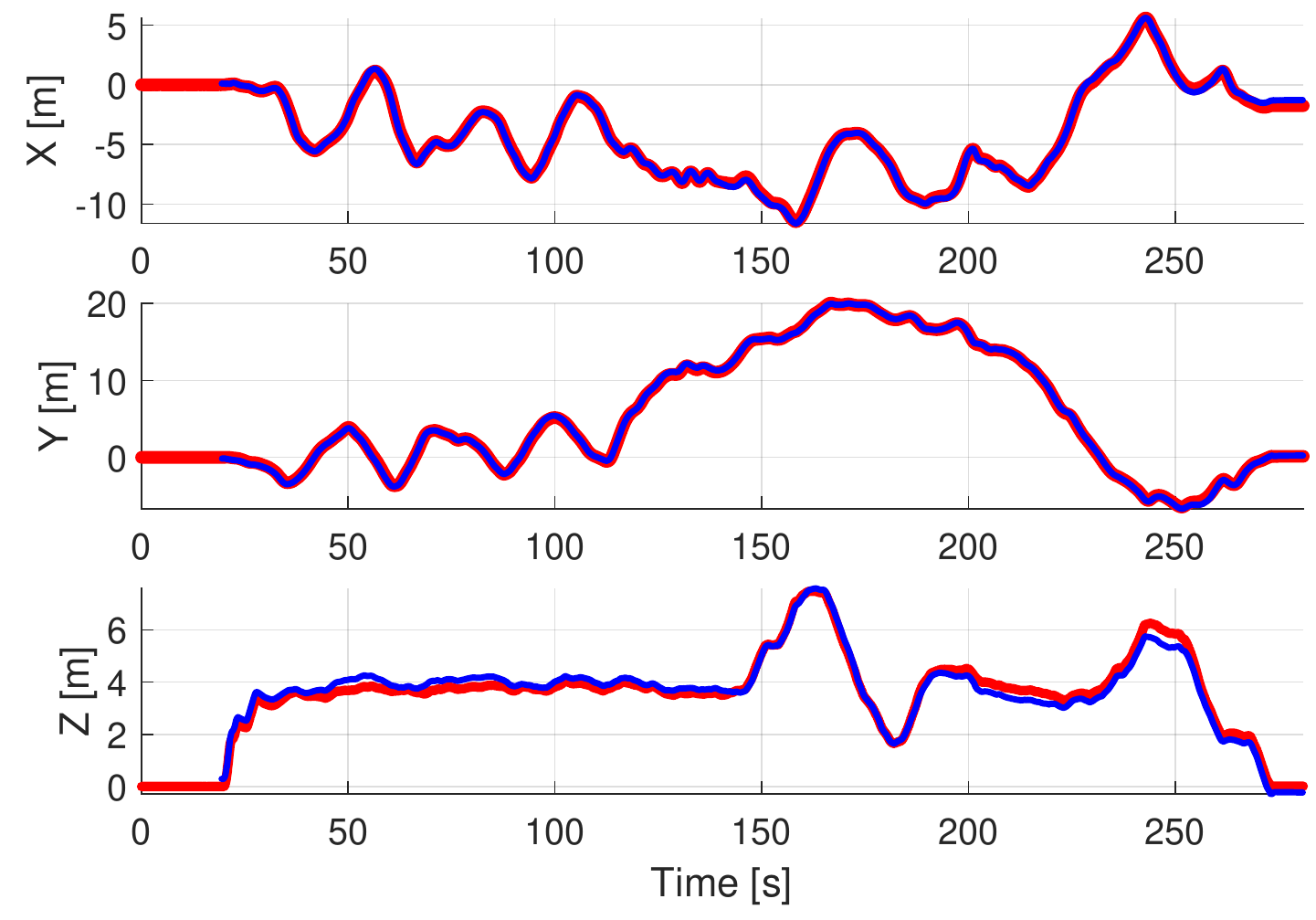}
} \vspace{3pt}
\subfloat[eee\_03] {
  \includegraphics[width=0.6\columnwidth]{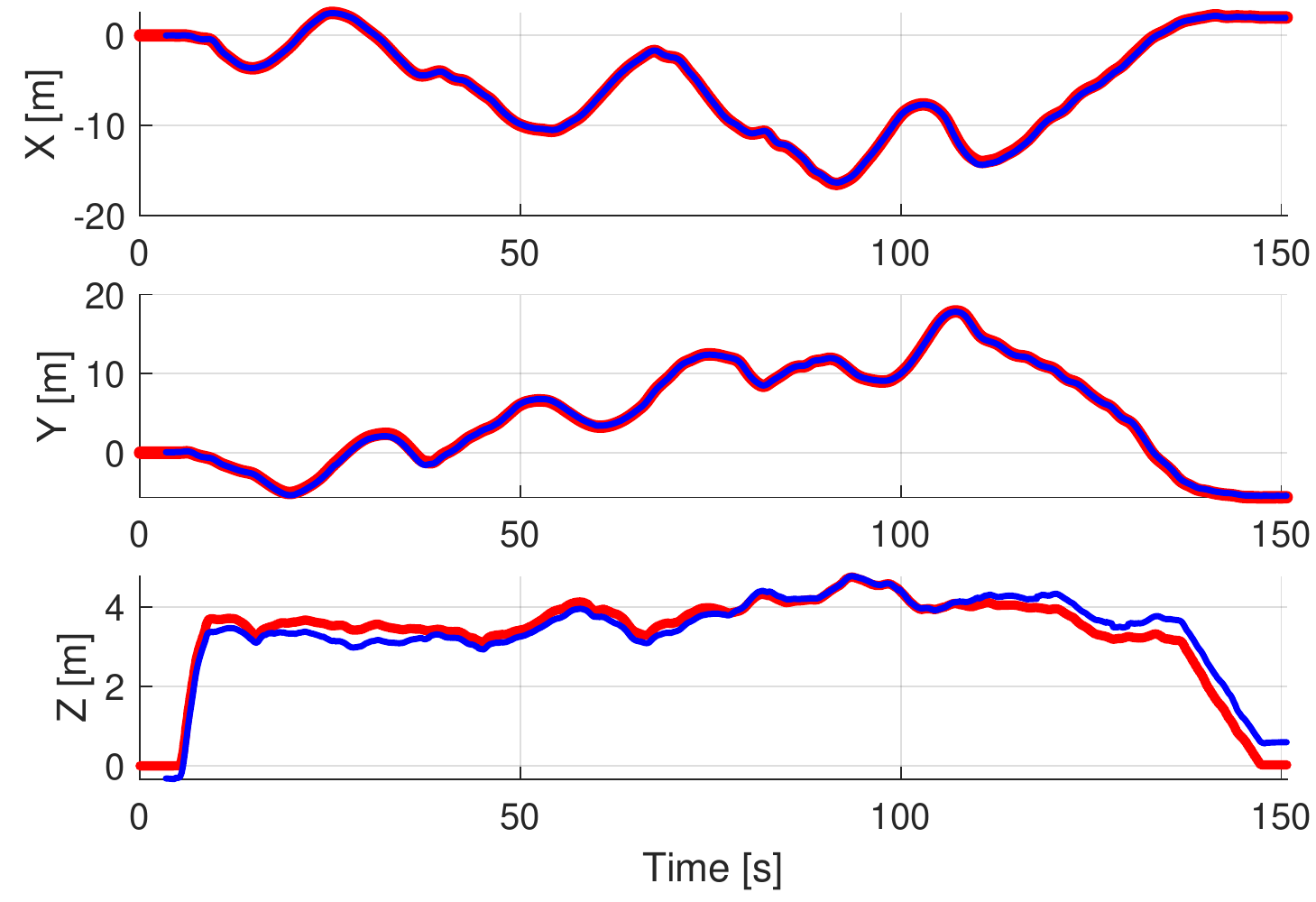}
} 
\hfill
\subfloat[nya\_01] {
  \includegraphics[width=0.6\columnwidth]{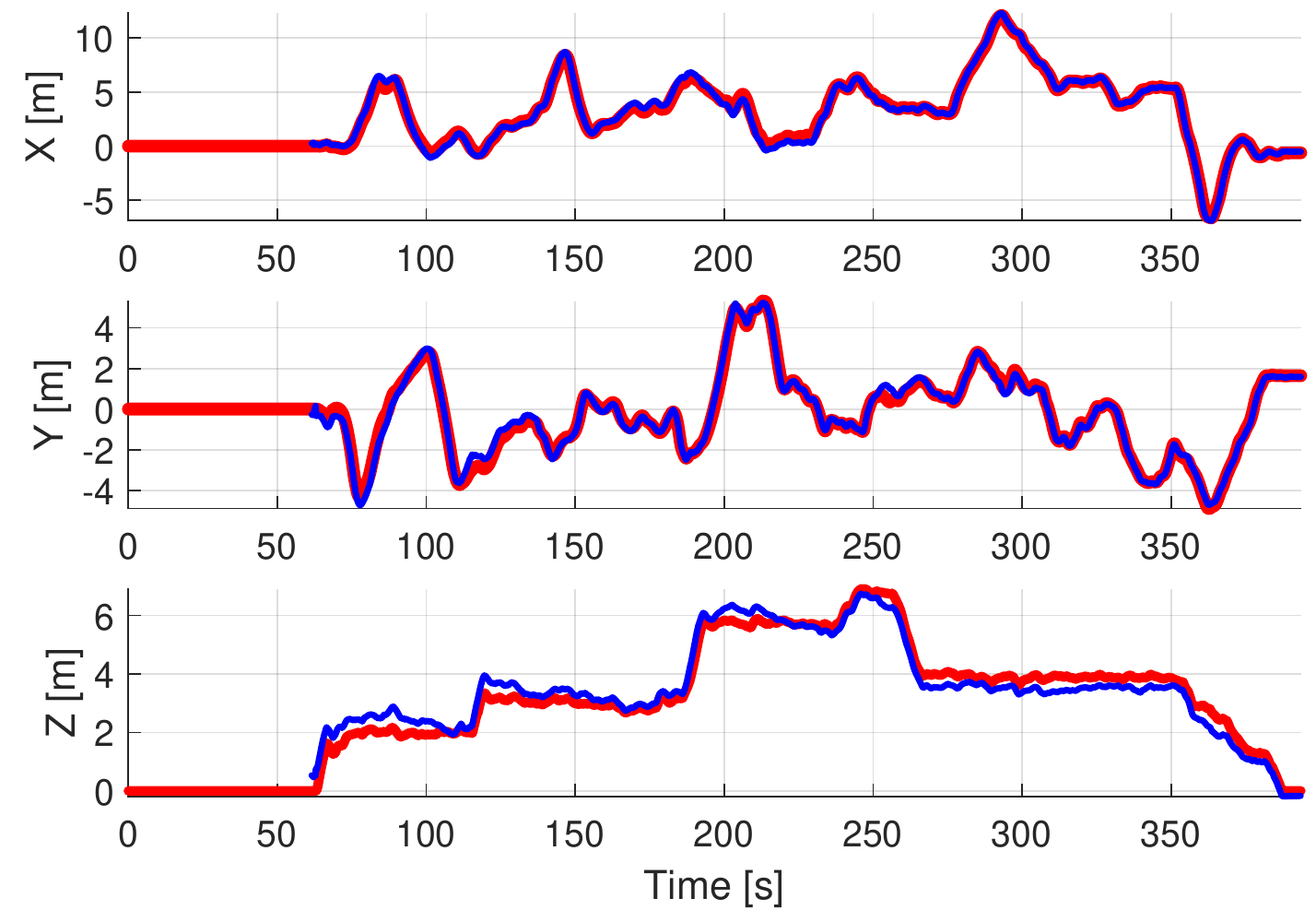}
} \vspace{3pt}
\subfloat[nya\_02] {
  \includegraphics[width=0.6\columnwidth]{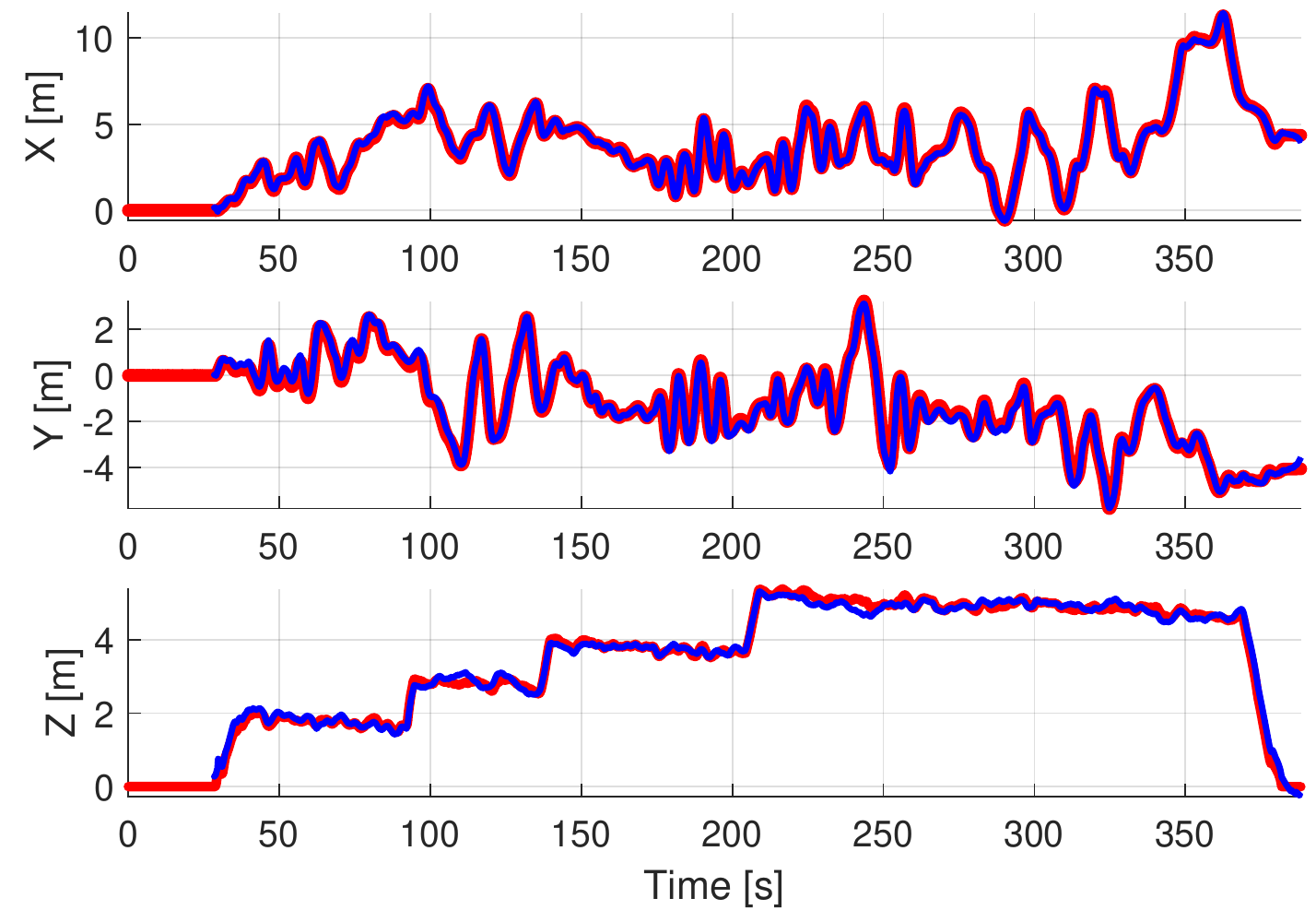}
} \vspace{3pt}
\subfloat[nya\_03] {
  \includegraphics[width=0.6\columnwidth]{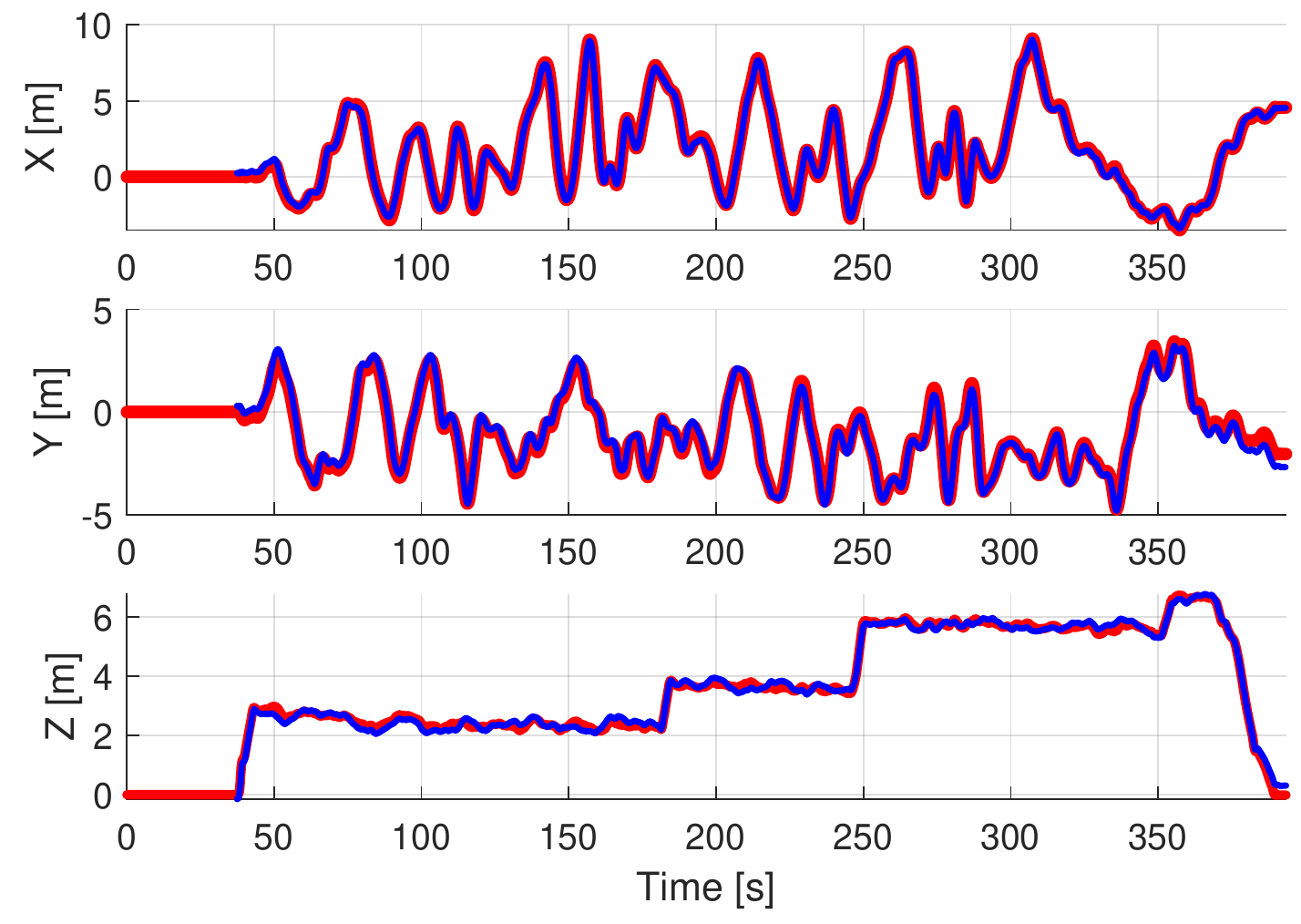}
} 
\captionsetup{font={footnotesize }}
\caption{Trajectories estimated by the proposed \textit{FEJ-VIRO} (blue) and the corresponding groundtruth (red) of \textit{six} runs available in the VIRAL dataset. }
\label{fig:real-world-ape}
\end{figure*}

\begin{figure}
  \centering 
  \subfloat[VIRO] {
    \includegraphics[width=0.9\columnwidth,trim=0 0 0 0,clip]{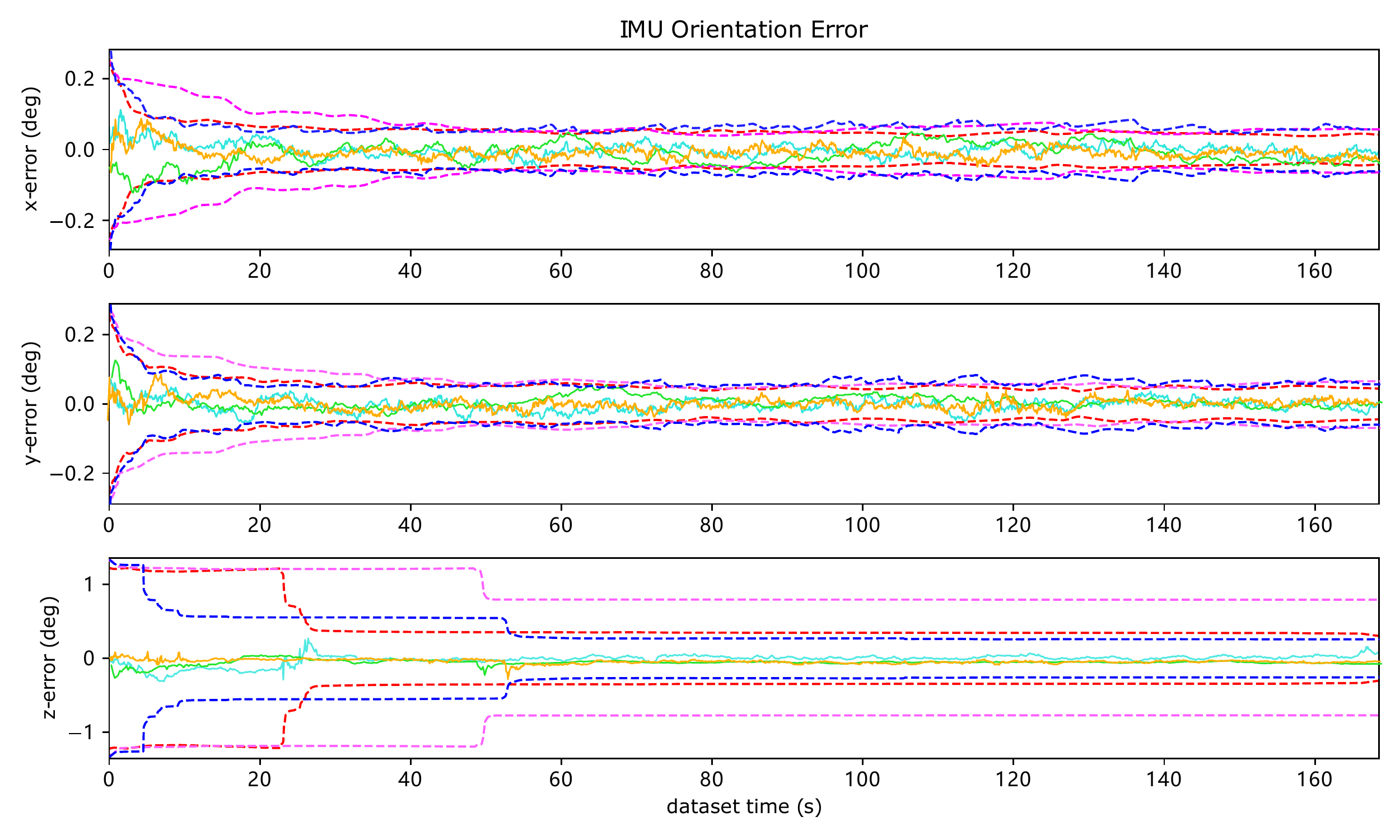}
  } 
  \newline
  \subfloat[FEJ-VIRO] {
    \includegraphics[width=0.9\columnwidth,trim=0 0 0 0,clip]{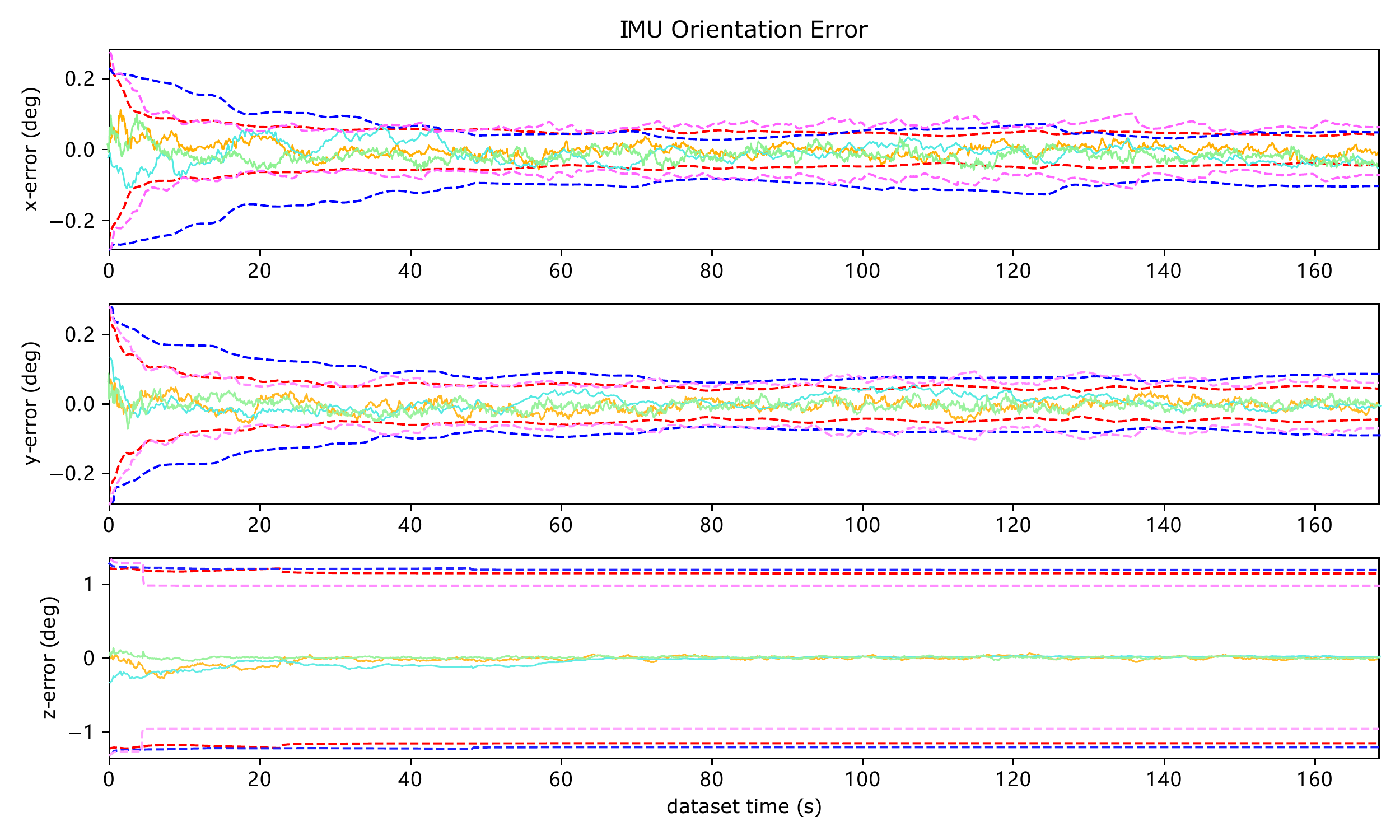}
  } 
  \captionsetup{font={footnotesize }}
\caption{Orientation error and $3\mbox{-}\sigma$ bounds of \textit{VIRO} and \textit{FEJ-VIRO} tested on three different simulation datasets. }
 \label{fig:3-sigma}
\end{figure}

We compare the NEES of \textit{OpenVINS}, \textit{VIRO} and \textit{FEJ-VIRO} to validate the inconsistency issue described in \ref{subsec:consistent-issue} and evaluate the consistency of the proposed \textit{FEJ-VIRO}. The NEES results are listed in Fig.\ref{fig:nees}
Note that we enabled FEJ for \textit{OpenVINS}, and we only disabled \textit{VIRO}'s FEJ when updating state with ranging measurements, so the inconsistency of \textit{VIRO} results only from state update with UWB ranging measurements. 

Fig.\ref{fig:nees} shows that the standard updates with UWB ranging measurements results in large NEES. The dotted line in Fig.\ref{fig:nees} represents the 3 degree of freedoms of the position and orientation. The OpenVINS exceeds the dotted line sometimes because the visual-inertial estimator degrades when the platform is moving with constraints (e.g. planar motion). By leveraging FEJ technique, the inconsistency issue is solved and results in lower NEES of \textit{FEJ-VIRO}. To further evaluate the proposed long-short window structure, we compare it with the method with short-only sliding window for UWB initialization (see \textit{FEJ-VIRO-S} in Fig.\ref{fig:nees}). In the \textit{FEJ-VIRO-S}, the long-window poses in \eqref{eq:msckf-long-window} are marginalized out from the state, so we can not estimate the covariance by performing \eqref{eq:long-window-residual} to \eqref{eq:long-window-Jacobian}. The \textit{FEJ-VIRO-S} initializes the $\mathbf P_{aa}$ in \eqref{eq:long-window-paa} with the covariance provided by \textit{Cere-Solver} and initialize the cross-covariance $\mathbf P_{xa}$ with \textit{zero}. Note that the estimator fails if we initialize the UWB anchors' \textit{positions} with only poses in the short-window.

We know from Sec.\ref{subsec:consistent-issue} that when we linearize the ranging measurement models at the current state estimates, the number of unobservable directions decreases from four to three, where the rotation about the gravity becomes (wrongly) observable. To further validate our theory, we plot the $3\mbox{-}\sigma$ bounds of \textit{VIRO} and \textit{FEJ-VIRO} (shown in Fig.\ref{fig:3-sigma}) tested on three simulation datasets. Fig.\ref{fig:3-sigma} shows that the $3\mbox{-}\sigma$ (uncertainty) of the z-axis rotation estimated by \textit{VIRO} reduces rapidly after the UWB anchors are initialized. However, we have proved in Sec.\ref{sec:observability-study} that the rotation about the gravity is unobservable for visual-inertial-ranging estimator, so the VIRO is over-confident about the estimated z-axis rotation. The over-confidence is the main cause of inconsistency.

Fig.\ref{fig:3-sigma} shows that the $3\mbox{-}\sigma$ bounds of z-axis rotation estimated by \textit{FEJ-VIRO} do not degrade. By linearize the model and ranging measurement models at the first estimated state, the number of unobservable directions of the estimated system is the same as that of the ideal system. The \textit{FEJ-VIRO} is not over-confident about the estimated z-axis rotation, which is the key to the consistency and low NEES shown in Fig.\ref{fig:nees}.


\begin{table}
\centering
\caption{Average ATE (m) of \textit{six} runs in VIRAL Dataset}\label{tab:ATE}
\renewcommand\arraystretch{1.3}{
\begin{tabular}{ccccc}
\toprule
        & OpenVINS    & VINS-Fusion      & VIRO & VIRO \\
        &             & (RT / LC)        &      & (FEJ) \\
\hline
eee\_01 & 0.647 & 0.564 / \textit{0.364} & 0.824 & \textbf{0.479} \\
eee\_02 & 0.604 & 0.582 / \textit{0.252} & 0.506 & \textbf{0.340} \\
eee\_03 & 0.417 & 0.540 / {0.390} & 0.350 & \textbf{0.348} \\
nya\_01 & 0.741 & 0.538 / \textit{0.235} & 0.552 & \textbf{0.505} \\
nya\_02 & 0.478 & 0.375 / {0.331} & 0.368 & \textbf{0.202} \\
nya\_03 & 0.692 & 0.748 / {0.370} & 0.567 & \textbf{0.290} \\
\toprule
\end{tabular} }
\end{table}

\subsection{Real-world Experiments}
We further evaluate our method based on a challenging real-world dataset, VIRAL\cite{nguyen2021ntuviral}, which provides measurements from extensive sensors, such as IMU, stereo cameras, and three UWB anchor nodes.

We test \textit{VIRO} and \textit{FEJ-VIRO} on six trajectories available in the VIRAL dataset. The only difference of \textit{VIRO} and \textit{FEJ-VIRO} is the linearization point of UWB ranging measurement models. We use 11 clones and 200 features for the real-time estimation. We incorporate ranging measurements from three UWB anchors and echo ranging measurements between UWB anchor nodes. For UWB initialization, we select key-frame poses every 0.3m to build the long window, and perform UWB initialization (see Sec.\ref{sec:UWB-init}) when there are over 50 poses in the long window. 

We compare \textit{VIRO} and \textit{FEJ-VIRO} with the original \textit{OpenVINS} to validate that our methods improve the accuracy and reduce localization drifts of VIO during long distance trajectories.
We also compare our methods with \textit{VINS-Fusion}\cite{qin2019a}\cite{qin2019b}, which is the state-of-the-art VI-SLAM estimator based on sliding-window smoothing. As discussed in Sec.\ref{sec:related_work}, the \textit{VINS-Fusion} leverage global information (loop-closure) and global non-linear optimization to achieve bounded error with significantly higher computational complexity. 
We use the configure files provided by \cite{viral_site} with limited modifications for experiments.

We record the trajectory of these methods and evaluate with the methods recommended by \cite{nguyen2021ntuviral}. The \textit{OpenVINS} configurations are available on the home page of VIRAL\cite{viral_site}. All experiment results are listed in Tab.\ref{tab:ATE}. \textit{VINS-Fusion} outputs the \textit{optimized path} after loop-closure and the \textit{real-time pose estimates}. Since poses of the \textit{optimized path} utilize \textit{future information} to improve accuracy and smoothness, which is non-causal and can not be used by motion controllers in practice, we compare our methods with the \textit{real-time pose estimates} and only list the accuracy of \textit{optimized path} for reference. We set to bold the best accuracy among \textit{OpenVINS}, \textit{VINS-Fusion (RT)}, \textit{VIRO} and the proposed \textit{FEJ-VIRO}. We also set the \textit{VINS-Fusion (LC)} to italic if it outperforms \textit{FEJ-VIRO} although it is listed only for reference.
The column of \textit{VINS-Fusion (RT)} represents the \textit{real-time pose estimates}, while the column of \textit{VINS-Fusion (LC)} represents the \textit{optimized path} after loop-closure.
We also plot the estimated trajectories by \textit{FEJ-VIRO} as well as the groundtruth in Fig.\ref{fig:real-world-ape}. From \ref{tab:ATE}, our \textit{FEJ-VIRO} is accurate and is nearly as accurate as the \textit{optimized path} of \textit{VINS-Fusion}. Note that our methods only requires an UWB anchor initialization at the beginning and an FEJ-EKF update during the running.
On the other hand, although the \textit{optimized path} of \textit{VINS-Fusion} is in high accuracy, the \textit{real-time pose estimates} suffer from sudden changes due to the loop-closure adjustment. While our method does not result in \textit{jump} of the \textit{real-time pose estimation}, thus is a better approach to on-board estimator that collaborate with motion controllers.

\section{CONCLUSIONS}

In this paper, we analyze the observability of visual-inertial-ranging estimator and propose a consistent filter incorporating measurements from cameras, IMU, and multiple ultra-widebands (UWB). In particular, we prove that \textit{four} unobservable directions exist in the VIRO system, which means that we can not design drift-free odometry by fusing measurements from UWB, cameras and IMU. We further study the structure of the estimated system to show that there are only three unobservable directions when we evaluate the Jacobians at the latest state estimates during every time step, which result in inconsistency. Based on these analyses, we leverage the FEJ technique to fuse UWB measurements consistently in a tightly-coupled MSCKF framework. Finally, we validate our analysis and the proposed system with both simulation and real-world experiments.

For future works, we would try other methods to address the inconsistency issue such as invariant filter or observability constraint (OC) technique to achieve higher accuracy. 

\section*{APPENDIX}

\textbf{Ideal Observability Properties:}
Following the previous work\cite{Consistency_ana}, it is easy to verify that $\mathbf N_o$ in \textit{theorem}.\ref{thm:nullspace} spans the null-space of $\mathbf H_{v,k} \bm \Phi_{k,1}$. And it is obvious that the $\mathbf N_1$ in \textit{theorem}.\ref{thm:nullspace} spans a subspace of the left-zero space of $\mathbf H_{u,k} \bm \Phi_{k,1}$, so we only verify that $\mathbf N_2$ is also a null-space of of it.
\begin{align}
\mathcal M_k &= \mathbf H_{u,k} \bm \Phi_{k,1} \mathbf N_2 = \nonumber \\[3pt]
&\scalemath{.8}{ \begin{bmatrix}
\bm \Gamma_{11} & \bm \Gamma_{12} & -\bm \zeta_{a_1} \bm \phi_{53} & -\bm \zeta_{a_1} \bm \phi_{54} & -\bm \zeta_{a_1} & \mathbf 0 & \bm \zeta_{a_1} & \mathbf 0 \\
\bm \Gamma_{21} & \bm \Gamma_{22} & -\bm \zeta_{a_1} \bm \phi_{53} & -\bm \zeta_{a_2} \bm \phi_{54} & -\bm \zeta_{a_2} & \mathbf 0 & \mathbf 0 & \bm \zeta_{a_1} \\
\mathbf 0 & \mathbf 0 & \mathbf 0 & \mathbf 0 & \mathbf 0 & \mathbf 0 & \bm \Lambda_a & -\bm \Lambda_a
\end{bmatrix} } \nonumber \\[3pt]
&\times \scalemath{.85}{
\begin{bmatrix}
    {^{I_1}_G \mathbf R} \cdot {^G \mathbf g} \\
    \mathbf 0_3 \\
    -\lfloor {^G \mathbf v_{I_1}} \times \rfloor \cdot {^G \mathbf g} \\
    \mathbf 0_3 \\
    -\lfloor {^G \mathbf p_{I_1}}\times \rfloor \cdot {^G \mathbf g} \\
    -\lfloor {^G \mathbf f}\times \rfloor \cdot {^G \mathbf g} \\
    -\lfloor {^G \mathbf p_{a_1}} \times \rfloor {^G \mathbf g} \\
    -\lfloor {^G \mathbf p_{a_2}} \times \rfloor {^G \mathbf g}
\end{bmatrix} } = \begin{bmatrix}
    \mathcal M_{k,1} \\
    \mathcal M_{k,2} \\
    \mathcal M_{k,3}
\end{bmatrix}
\end{align}
Next we will show that $\mathcal M_{k,1} = \mathcal M_{k,2} = \mathcal M_{k,3} = 0$. Subscribe the analytical form of \eqref{eq:UWB-obs-matrix} in the above equation, we have:
\begin{align} \label{eq:proof1}
    \mathcal M_{k,1} =&{} \bm \Gamma_{11} {^{I_1}_G \mathbf R} \cdot {^G \mathbf g} 
    + \bm \zeta_{a_1} \bm \phi_{53} \lfloor {^G \mathbf v_{I_1}} \times \rfloor \cdot {^G \mathbf g} \\[3pt]
    & + \bm \zeta_{a_1} \lfloor {^G \mathbf p_{I_1}}\times \rfloor \cdot {^G \mathbf g} + \bm \zeta_{a_1} \lfloor {^G \mathbf p_{a_1}} \times \rfloor {^G \mathbf g} \nonumber\\[3pt]
    =&{} \scalemath{.9} {
    \bm \zeta_{a_1} \cdot (
    ({^{I_k}_G \mathbf R^\top} \lfloor {^I \mathbf p_r} \times \rfloor \cdot \phi_{11} - \phi_{51}) {^{I_1}_G \mathbf R} {^G \mathbf g} +} \nonumber \\[3pt]
    & \scalemath{.9} {
    (k-1) \lfloor {^G \mathbf v_{I_1}} \times \rfloor {^G \mathbf g} + \lfloor {^G \mathbf p_{I_1}} \times \rfloor {^G \mathbf g} - \lfloor {^G \mathbf p_{a_1}} \times \rfloor {^G \mathbf g} 
    )} \nonumber \\[3pt]
    =& \scalemath{.9} {
    \bm \zeta_{a_1} \cdot ({^{I_k}_G \mathbf R^\top} \lfloor {^I \mathbf p_r} \times \rfloor {^{I_k}_G \mathbf R} {^G \mathbf g} - \nonumber} \\[3pt]
    &\scalemath{.9} {
    \lfloor {^G \mathbf p_{I_1}} + (k-1) {^G \mathbf v_{I_1}} - \frac{(k-1)^2}{2}{^G \mathbf g} - {^G \mathbf p_{I_k}} \times \rfloor {^G \mathbf g} + } \nonumber \\[3pt]
    &\scalemath{.9} {
    (k-1) \lfloor {^G \mathbf v_{I_1}} \times \rfloor + \lfloor {^G \mathbf p_{I_1}} \times \rfloor {^G \mathbf g} - \lfloor {^G \mathbf p_{a_1}} \times \rfloor {^G \mathbf g} )} \nonumber \\[3pt]
    =&\scalemath{.9} { \bm \zeta_{a_1} \cdot 
        (\lfloor {^{I_k}_G \mathbf R^\top} {^I \mathbf p_r} \times \rfloor {^G \mathbf g} + \lfloor {^G \mathbf p_{I_k}} \times \rfloor {^G \mathbf g} - \lfloor {^G \mathbf p_{a_1}} \times \rfloor {^G \mathbf g})
    } \nonumber \\[3pt]
    =& \scalemath{.9} { \bm \zeta_{a_1} \cdot 
        \lfloor {^{I_k}_G \mathbf R^\top} {^I \mathbf p_r} + {^G \mathbf p_{I_k}} - {^G \mathbf p_{a_1}} \times \rfloor  {^G \mathbf g}
    } \nonumber
\end{align}
Because $\bm \zeta_{a_1} = -2({^{I_k}_G \mathbf R^\top} {^I \mathbf p_r} + {^G \mathbf p_I} - {^G \mathbf p_a})^\top$, so we have $\mathcal M_{k,1} = \frac{1}{2} \lfloor \bm \zeta_{a_1}^\top \times \rfloor \bm \zeta_{a_1}^\top = 0$. Similarly, we have $\mathcal M_{k,2} = 0$.

The last step is to evaluate $\mathcal M_{k,3}$:
\begin{align}
    \mathcal M_{k,3} =& - \Lambda_a \cdot \lfloor {^G \mathbf p_{a_1}} \times \rfloor + \Lambda_a \cdot \lfloor {^G \mathbf p_{a_2}} \times \rfloor \\[3pt]
    =& -2({^G \mathbf p_{a_1}} - {^G \mathbf p_{a_2}})^\top (\lfloor {^G \mathbf p_{a_1}} \times \rfloor -  \lfloor {^G \mathbf p_{a_1}} \times \rfloor) \nonumber \\[3pt]
    =& 0 \nonumber
\end{align}
As a result, $\mathbf N_o$ spans the null-space of the visual-inertial-ranging state estimator.

\textbf{Actual Observability Properties:}
Under the actual case, the UWB measurement models are linearized at the ever-changing estimated values, which makes the last step of \eqref{eq:proof1} become:
\begin{align} \label{eq:actual-proof}
    \hat{\mathcal M}_{k,1} &= \scalemath{.9} { \bm \zeta_{a_1} \cdot 
        (\lfloor {^{I_k}_G \hat{\mathbf R}^\top} {^I \mathbf p_r} \times \rfloor {^G \mathbf g} + \lfloor {^G \hat{\mathbf p}_{I_k}} \times \rfloor {^G \mathbf g} - \lfloor {^G \mathbf p_{a_1}} \times \rfloor {^G \mathbf g})
    } 
\end{align}
Note that in \eqref{eq:actual-proof}, ${^{I_k}_G \hat{\mathbf R}^\top} {^I \mathbf p_r} + {^G \hat{\mathbf p}_{I_k}} = {{^G \hat{\mathbf p}_r}} = {^G \mathbf p_r} - {^G \tilde{\mathbf p}_r}$, an error term appears to destroy the observability, which results in $\hat{\mathcal M}_{k,1} \neq 0$. So the unobservability direction spanned by $\mathbf N_2$ does not hold in the actual case.


\bibliographystyle{IEEEtran}
\bibliography{bibfiles/bibs.bib}

\end{document}